\documentclass{article}

\usepackage{comment}

\usepackage{tikz,siunitx}
\usetikzlibrary{decorations.pathreplacing}
\usetikzlibrary{shapes,arrows}
\usepackage{natbib}

\date{}

\usepackage[utf8]{inputenc} %
\usepackage[T1]{fontenc}    %
\usepackage{lmodern}
\usepackage{hyperref}       %
\usepackage{url}            %
\usepackage{booktabs}       %
\usepackage{amsfonts}       %
\usepackage{nicefrac}       %
\usepackage{microtype}      %

\usepackage{amsmath}
\usepackage{amsthm}
\usepackage{amssymb}

\usepackage{bm}
\usepackage[title]{appendix}
\usepackage{courier}
\usepackage{enumitem}
\usepackage{graphicx}
\usepackage{algorithmic}
\usepackage[caption=false,font=footnotesize]{subfig}
\usepackage{float}
\usepackage{longtable}
\usepackage[margin=1in]{geometry}
\usepackage{authblk}
\usepackage{nicefrac}
\usepackage{mathtools}
\usepackage{color}
\usepackage{xspace} %
\usepackage{times}
\usepackage{stmaryrd}

\usepackage{todonotes}

\newtheorem*{theorem*}{Theorem}
\newtheorem*{lemma*}{Lemma}

\RequirePackage{latexsym}
\RequirePackage{amsthm}
\RequirePackage{amssymb}
\RequirePackage{bm}
\RequirePackage[mathscr]{euscript}

\newcommand{\ts}{\mathsf{T}}

\newcommand{\RR}[2]{\mathbb{R}^{#1 \times #2}} %

\newcommand{\vectorize}{\mathrm{vec}}
\newcommand{\bigo}{\mathcal{O}}

\newcommand{\St}{\bm{\mathcal{S}}}
\newcommand{\At}{\bm{\mathcal{A}}}
\newcommand{\Bt}{\bm{\mathcal{B}}}
\newcommand{\Xt}{\bm{\mathcal{X}}}

\newcommand{\Gt}{\bm{\mathcal{G}}}
\newcommand{\Tt}{\bm{\mathcal{T}}}
\newcommand{\It}{\bm{\mathcal{I}}}
\newcommand{\Yt}{\bm{\mathcal{Y}}}
\newcommand{\Qt}{\bm{\mathcal{Q}}}

\newcommand{\Ht}{\bm{\mathcal{H}}}

\newcommand{\Wt}{\bm{\mathcal{W}}}
\newcommand{\Ct}{\bm{\mathcal{C}}}
\newcommand{\CP}[1]{\llbracket #1 \rrbracket}

\newcommand{\T}{\Tt}
\newcommand{\Rbb}{\RR}

\newcommand{\ab}{\mathbf{a}}
\newcommand{\bb}{\mathbf{b}}
\newcommand{\cbb}{\mathbf{c}}

\newcommand{\eb}{\mathbf{e}}

\newcommand{\gb}{\mathbf{g}}

\newcommand{\tb}{\mathbf{t}}
\newcommand{\ub}{\mathbf{u}}
\newcommand{\vb}{\mathbf{v}}

\newcommand{\xb}{\mathbf{x}}
\newcommand{\yb}{\mathbf{y}}
\newcommand{\zb}{\mathbf{z}}

\newcommand{\Ab}{\mathbf{A}}
\newcommand{\Bb}{\mathbf{B}}
\newcommand{\Cb}{\mathbf{C}}
\newcommand{\Db}{\mathbf{D}}

\newcommand{\Gb}{\mathbf{G}}

\newcommand{\Ib}{\mathbf{I}}

\newcommand{\Mb}{\mathbf{M}}

\newcommand{\Qb}{\mathbf{Q}}
\newcommand{\Rb}{\mathbf{R}}

\newcommand{\Tb}{\mathbf{T}}
\newcommand{\Ub}{\mathbf{U}}
\newcommand{\Vb}{\mathbf{V}}
\newcommand{\Wb}{\mathbf{W}}
\newcommand{\Xb}{\mathbf{X}}
\newcommand{\Yb}{\mathbf{Y}}

\renewcommand{\vec}[1]{\mathbf{\boldsymbol{#1}}}

\ifx\BlackBox\undefined
\newcommand{\BlackBox}{\rule{1.5ex}{1.5ex}}  %

\ifx\QED\undefined
\def\QED{~\rule[-1pt]{5pt}{5pt}\par\medskip}
\fi

\ifx\proof\undefined
\newenvironment{proof}{\par\noindent{\bf Proof\ }}{\hfill\BlackBox\\[2mm]}
\fi

\ifx\theorem\undefined
\newtheorem{theorem}{Theorem}
\fi

\ifx\example\undefined
\newtheorem{example}{Example}
\fi

\ifx\property\undefined

\fi

\ifx\lemma\undefined
\newtheorem{lemma}[theorem]{Lemma}
\fi

\ifx\proposition\undefined
\newtheorem{proposition}[theorem]{Proposition}
\fi

\ifx\remark\undefined
\newtheorem{remark}[theorem]{Remark}
\fi

\ifx\corollary\undefined
\newtheorem{corollary}[theorem]{Corollary}
\fi

\ifx\definition\undefined
\newtheorem{definition}{Definition}
\fi

\ifx\step\undefined
\newtheorem{step}{Step}
\fi

\ifx\conjecture\undefined
\newtheorem{conjecture}[theorem]{Conjecture}
\fi

\ifx\axiom\undefined
\newtheorem{axiom}[theorem]{Axiom}
\fi

\ifx\claim\undefined
\newtheorem{claim}[theorem]{Claim}
\fi

\ifx\assumption\undefined
\newtheorem{assumption}[theorem]{Assumption}
\fi

\newcommand{\CC}{\mathbb{C}} %
\newcommand{\EE}{\mathbb{E}} %
\newcommand{\PP}{\mathbb{P}} %
\renewcommand{\RR}{\mathbb{R}} %

\newcommand{\tr}{\mathop{\mathrm{tr}}}
\newcommand{\diag}{\mathop{\mathrm{diag}}}

\newcommand{\inner}[2]{\left\langle #1,#2 \right\rangle}
\newcommand{\norm}[1]{\left\lVert#1\right\rVert}

\newcommand{\Sigmab}{\mathbf{\Sigma}}

\newcommand{\kron}{\otimes}
\newcommand{\krao}{\odot}

\newcommand{\hadam}{*}
\newcommand{\outprod}{\circ}

\newcommand{\cprank}{\mathrm{rank}_{\mathrm{CP}}}
\newcommand{\rank}{\mathrm{rank}}

\newcommand\textequal[1]{\stackrel{\mathclap{\scalebox{.75}{\normalfont\mbox{#1}}}}{=}}

\renewcommand{\vec}[1]{\ensuremath{\mathbf{#1}}}

\newcommand{\mat}[1]{\ensuremath{\mathbf{#1}}}

\newcommand{\ten}[1]{\mat{\ensuremath{\boldsymbol{\mathcal{#1}}}}}

\newcommand{\ttm}[1]{\times_{#1}}

\newcommand{\edgelabel}[1]{\scalebox{0.7}{\textcolor{gray}{#1}}}
\DeclareMathOperator*{\TT}{\operatorname{TT}}

\newcommand{\dimcolor}{\textcolor{gray}}
\newcommand{\indcolor}{\textcolor{blue}}
\definecolor{reshapecolor}{HTML}{f5ef4c}

\usepackage{xcolor}
\hypersetup{
    colorlinks,
    linkcolor={red!50!black},
    citecolor={blue!50!black},
    urlcolor={blue!80!black}
}

\title{Tensor Cookbook: Mastering Tensors through Diagrams}

\newcommand\Mark[1]{\textsuperscript{#1}}

\author{Beheshteh T. Rakhshan\Mark{$^{1,2}$}, Guillaume Rabusseau\Mark{$^{1,2,3}$}\\
\Mark{1}Mila \& DIRO, \Mark{2} Université de Montréal, \Mark{3}CIFAR AI Chair}

\def\baseline{-0.5ex}

\usepackage{xparse}

\newif\ifcompiletikz

\begin{document}
\maketitle
\newpage
\tableofcontents
\clearpage
\newpage
\section*{Notations}
\begin{tabular}{ll}
      $a$ & Scalar; A rank-0 tensor
      \\
      $\vb$ & Column vector; A rank-1 tensor\\
    $\Mb$ & A matrix; equivalently, a rank-2 tensor or tensor of order 2\\
    $\Tt\in \RR^{d_1 \times \cdots\times d_N}$ & An $N$-th order tensor \\
    $\Gt\in \RR^{d_1 \times d_2\times d_3}$ & A third-order tensor \\
    $\Gt_{i,j,k}$ & The $(i,j,k)$-th element of $\Gt$ \\
    $\Gt_{i,:,k}$ & Mode-2 fiber of a third-order tensor $\ten{G}$, equivalent to a vector $\vb\in \RR^{d_2}$\\
    $\Gt_{(2)}$ & Mode-2 matricization of $\Gt$, equivalent to a matrix $\Mb\in \RR^{d_2 \times d_1 d_3}$ \\
    $\vectorize(\Gt)$ & A vectorization of $\Gt$ equivalent to a vector $\vb\in\RR^{d_1d_2d_3}$\\
    $\St\outprod \Tt$ & The tensor product (equiv. outer product) of tensors $\St$ and $\Tt$ \\
    $\inner{\Gt}{\Bt}$ & Inner product between tensors $\Gt$ and $\Bt$\\
    $\norm{\Gt}_F$  &  Frobenius norm of the tensor $\Gt$\\
    $\St\kron\Tt$ & The Kronecker product of tensors $\St$ and $\Tt$\\
         \\
    $\Mb^\ts$ & The transposed matrix of matrix $\Mb$\\
    $\Ib_d$ & The identity matrix of size $d$\\
    $\Mb\krao\Bb$ & The Khatri-Rao product of matrices $\Mb$ and $\Bb$\\
    $\Mb\hadam\Bb$ & The Hadamard product of matrices $\Mb$ and $\Bb$ \\ 
\end{tabular}
\newpage
\section*{Tensor Network Notations}

    \begin{tabular}{ll}
  Zero-order tensor~(scalar) &
    $\begin{tikzpicture}[baseline=-0.5ex]
    \tikzset{tensor/.style = {minimum size = 0.5cm,shape = circle,thick,draw=black,inner sep = 0pt}, edge/.style = {thick,line width=.4mm},every loop/.style={}}
    \def\y{0}
    \def\x{0}
    \node[tensor,fill=blue!20!white] (a) at (\x,\y){\scalebox{0.85}{$a$}};
    \end{tikzpicture}\in\RR$
    \\
    First-order tensor~(vector)
    & $\begin{tikzpicture}[baseline=-0.5ex]
    \tikzset{tensor/.style = {minimum size = 0.5cm,shape = circle,thick,draw=black,inner sep = 0pt}, edge/.style = {thick,line width=.4mm},every loop/.style={}}
    \def\y{0}
    \def\x{0}
    \node[tensor,fill=blue!40!red!60!white] (a) at (\x,\y){\scalebox{0.85}{$\vb$}};
    \draw[edge] (\x+0.25,\y) -- (\x+0.6,\y)  node [midway,above] {\scalebox{0.7}{\dimcolor{$d$}}};
\end{tikzpicture}\in\RR^d$
\\
Second-order tensor~(matrix) & $\begin{tikzpicture}[baseline=-0.5ex]
    \tikzset{tensor/.style = {minimum size = 0.5cm,shape = circle,thick,draw=black,inner sep = 0pt}, edge/.style = {thick,line width=.4mm},every loop/.style={}}
    \def\x{0}
    \def\y{0}
    \node[tensor,fill=red!50!white] (A) at (\x+1,\y){\scalebox{0.85}{$\Mb$}};
    \draw[edge] (\x+1.25,\y) -- (\x+1.6,\y)  node [midway,above] {\scalebox{0.7}{\dimcolor{$d_2$}}};
    \draw[edge] (\x+0.75,\y) -- (\x+0.4,\y)  node [midway,above] {\scalebox{0.7}{\dimcolor{$d_1$}}};
    \end{tikzpicture}\in\RR^{d_1\times d_2}$\\
\end{tabular}

\newpage
\section*{Introduction}
High–dimensional data arise naturally in many areas of science and engineering, including machine learning, signal processing, computational physics, and statistics. Such data are often represented as tensors, i.e., multi–dimensional generalizations of matrices. While tensors provide a natural representation for multi–modal structure, their direct manipulation quickly becomes challenging as the order of the tensor grows: the number of parameters increases exponentially, and algebraic expressions involving many indices become difficult to interpret and implement.

Tensor networks (TNs) provide an effective framework for addressing these challenges.
Originally introduced by~\cite{penrose1971applications} and extensively developed in the quantum physics community, the graphical language of tensor networks encodes tensor contractions as edges in a graph, a formalism that not only reduces notational overhead but also reveals structural properties of tensor operations that are obscured by index notation. Despite the central role of high-dimensional tensors in modern machine learning and numerical analysis, tensor network diagrams remain underutilized outside of quantum computing, partly due to the lack of a self-contained mathematical reference accessible to a broad technical audience.
The goal of this manuscript is to provide a self-contained guide to tensor networks and their use in tensor algebra. We present the main operations on tensors, such as contractions, products, and reshaping, through graphical notation, and show how classical tensor decompositions and related computations can be naturally expressed in this framework. In addition, we illustrate how tensor networks simplify the derivation of gradients and the manipulation of high-dimensional probability distributions.

Throughout, we demonstrate that the diagrammatic approach does not merely simplify notation; it yields genuinely shorter and more transparent proofs of classical identities, rank bounds, and gradient formulas that would otherwise require laborious index manipulation.  We emphasize that the results presented here are not new; all of the theorems are classical, with most drawn from~\citep{kolda2009tensor, kolda2006multilinear, oseledets2011tensor} and the remainder from sources cited throughout the manuscript. Our aim is instead to show how these familiar results can be derived and understood more naturally through tensor diagrams.

\subsection*{Organization}

This manuscript is organized into five chapters, each building on the previous to develop 
a complete and self-contained treatment of tensor networks.

\medskip
\noindent\textbf{Chapter 1: Tensor Network Basics.}
We introduce the graphical language of tensor networks from first principles. Tensors are 
represented as nodes with labeled legs, and contractions are encoded as shared edges. We 
develop the core operations: inner products, outer products, traces, and Frobenius norms within this framework, demonstrating how familiar matrix identities extend naturally to 
higher-order tensors. The chapter concludes with copy tensors (hyperedges) and graphical 
representations of matrix factorizations, including QR decomposition and singular value 
decomposition (SVD).

\medskip
\noindent\textbf{Chapter 2: Operations on Tensors.}
We treat tensor reshaping operations, fibers, slices, matricization, and vectorization, formally and graphically. We then develop the principal tensor products: the mode-$n$ 
product, the Kronecker product, the Khatri-Rao product, and the Hadamard product, each accompanied by diagrammatic proofs of their key identities. The chapter closes with a general theorem bounding the rank of tensor matricizations via cuts in the tensor network graph, unifying several classical rank inequalities under a single framework.

\medskip
\noindent\textbf{Chapter 3: Tensor Decompositions.}
We present the three principal decomposition models for high-dimensional tensors. The CANDECOMP/PARAFAC (CP) decomposition expresses a tensor as a sum of rank-one terms; the Tucker decomposition generalizes this via a compressed core tensor and factor matrices,
and the Tensor Train (TT) decomposition provides a chain-structured factorization whose parameter count scales linearly with tensor order. For each model, we provide the mathematical formulation, its tensor network diagram, rank characterizations, and computational algorithms, including Higher-Order SVD (HOSVD) for Tucker and the TT-SVD and Alternating Least Squares (ALS) algorithms for TT. We further discuss efficient algebraic operations in TT format and briefly survey advanced architectures, including Tensor Ring, Projected Entangled Pair States (PEPS), and Hierarchical Tucker decompositions.

\medskip
\noindent\textbf{Chapter 4: Computing Gradients of Tensor Networks.}
We develop a graphical calculus for the differentiation of tensor network functions. Starting 
from classical definitions of gradients and Jacobians, we establish that the Jacobian of a tensor network with respect to any core tensor is obtained simply by removing that core from the diagram. This yields transparent, diagram-level proofs of classical derivative 
identities, including gradients of quadratic forms and trace functionals, and extends naturally to networks where a core appears multiple times. We illustrate the framework with two applications: gradient-based optimization of 
the CP decomposition, and backpropagation through weight matrices parameterized in 
Matrix Product Operator (MPO) format.

\medskip
\noindent\textbf{Chapter 5: Probability Distributions and Random Tensor Networks.}
We explore two interconnected applications of tensor networks to probability and 
statistics. First, we show how joint probability distributions over discrete random 
variables can be encoded as tensors, and how marginals and conditionals reduce to simple 
diagrammatic operations. We then discuss how the Tensor Train format provides an 
efficient, structured parameterization of high-dimensional distributions, including 
Born Machine-style representations rooted in quantum physics. Second, we compute 
the expectations of random tensor networks whose entries are drawn from the standard normal 
distribution. Using Isserlis' theorem and diagrammatic reasoning, we derive closed-form 
expressions for quantities such as $\EE[\|\Ab\Bb\|_F^2]$ and 
$\EE[(\Ab^\ts \Ab)^{\kron 2}]$ with a brevity that would be difficult to achieve through conventional index-based methods.

\bigskip
\subsection*{Target Audience}

This manuscript is intended as a self-contained reference for researchers and 
practitioners who work with high-dimensional data and wish to develop fluency with tensor 
networks. The material is broadly accessible: a working knowledge of linear algebra 
at the level of a first undergraduate course, including matrix multiplication, 
QR decomposition, and SVD are sufficient for the majority of the text.

The presentation is designed to serve multiple communities simultaneously. For 
\textit{mathematicians}, the manuscript provides rigorous 
definitions, rank characterizations, and algorithm analyses grounded in linear algebra 
theory. For \textit{researchers in machine learning and data science}, it offers a practical 
toolkit to parameterize high-dimensional weight matrices, compute gradients through 
structured tensor decompositions, and compactly model complex distributions. For 
\textit{physicists and quantum computing researchers}, it bridges the diagrammatic 
conventions of tensor networks with the language of applied mathematics. 

All tensor network diagrams in this manuscript are typeset using \LaTeX\ with the \texttt{TikZ} 
package, and the diagrammatic notation is gradually introduced so that no prior exposure 
to tensor network diagrams is assumed.

\section{Tensor Networks Basics}\label{chapter:basics}

As their order increases, representing and working with tensors becomes more complicated. 
Tensor networks provide an efficient framework for working with these high-order objects, simplifying both their representation and analysis. 
The graphical notation of tensor networks offers an intuitive way to visualize and simplify complex tensor operations~\citep{orus2014practical,biamonte2017tensor}. 
In this section, we introduce basic notions on tensors and common tensor operations using the language of tensor networks.

\subsection{What are Tensor Networks?}
As their name suggests, Tensor Networks (TNs) are simply tensors connected to each other to form a network. More precisely, a tensor network is a graph in which vertices represent \emph{tensors} and edges represent \emph{contracted} indices (shared modes). The graphical notation for tensor networks was first introduced by~\cite{penrose1971applications}. Tensors are represented by shapes with legs~(edges), i.e.,
$
\begin{tikzpicture}[baseline=-0.5ex]
    \tikzset{tensor/.style = {minimum size = 0.5cm,shape = circle,thick,draw=black,inner sep = 0pt}, edge/.style = {thick,line width=.4mm},every loop/.style={}}
    \def\x{6}
    \node[tensor,fill=green!50!red!50!white] (T) at (\x,0) {\scalebox{0.85}{$\Tt$}};
    \draw[edge] (T) -- (\x-0.6,0);
    \draw[edge] (T) -- (\x+0.6,0);
    \draw[edge] (T) -- (\x-0.5,-0.3);
    \draw[dotted] (T) -- (\x,-0.6);
    \draw[dotted] (T) -- (\x+0.5,-0.3);
    \draw[dotted] (T) -- (\x+0.3,-0.5);
    \end{tikzpicture}.
$
We will use different shapes, such as rectangles, triangles, or circles, and various colors to represent tensors. Shapes and colors serve mainly as visual cues. 
The \emph{order} of a tensor is its number of dimensions, e.g., an $N$-th order tensor $\Tt\in\RR^{d_1\times\cdots\times d_N}$ is a multidimensional array of scalars~$\Tt_{i_1,\dots,i_N}$ indexed by $N$ indices $i_1\in [d_1], i_2\in[d_2],\cdots, i_N \in [d_N]$. Each axis is called a \emph{mode} of a tensor: an $N$-th order tensor has $N$ modes.
 
 Tensors naturally generalize vectors and matrices to higher-order arrays. As the order increases, representing these arrays becomes more challenging. \emph{Tensor networks} provide a simple and intuitive way to represent and manipulate these higher-order objects. Complex operations on tensors can be represented more easily with graphical notations of tensor networks~\citep{biamonte2017tensor,orus2014practical}.

\paragraph{Tensor Network Nodes}\label{def:tn-nodes}
In a tensor network, each node~(or vertex) represents a tensor, and the number of legs~(incident edges) corresponds to the order of the tensor: scalars are nodes with no edges, vector nodes have a single edge, matrix nodes have two edges, and so on. For example,
\begin{center}

    $\in \RR^{d_1\times d_2\times d_3\times d_4}$.

\end{center}
represent a scalar, a vector, a matrix, a third- and a fourth-order tensors, respectively. 
Throughout this manuscript,
\begin{itemize}
    \item Tensors are represented by colored shapes called nodes, where the colors and shapes have no specific meaning~(unless stated otherwise). 
    \item We use lower case letters for scalars~(e.g., $a$), bold lower case letters for vectors~(e.g., $\vb$), bold upper case letters for matrices (e.g., $\Mb$), and 
    bold upper case calligraphic letters for higher-order tensors~(e.g., $\Tt$).
    \item The dimension of each mode of a tensor is depicted by a gray letter positioned at the top of the corresponding leg, e.g., 
$
\in\RR^{d_1\times d_2\times d_3},
$
is a third-order tensor of size $d_1\times d_2\times d_3$.
    
$
\Tt_{i,j,k} = 
    %
$    
is the $(i,j,k)$-th element of the third-order tensor $\Tt$.
\end{itemize}
Note that large matrices can sometimes be viewed as high-order tensors through reshaping. For example, consider a matrix $\Mb \in \Rbb^{I \times J}$ where $I=I_1I_2I_3I_4$ and $J=J_1J_2J_3J_4$:
\begin{align*}
    &\Mb
    =
    %
\in\RR^{I_1I_2I_3I_4\times J_1J_2J_3J_4}.
\end{align*}
We can reshape this matrix to obtain a 4-th order tensor $\Tt$:
\begin{align*}
    \Tt = 
    %
    \in\RR^{I_1\times I_2\times I_3\times I_4\times J_1 \times J_2\times J_3\times J_4}.
\end{align*}
In this manuscript, we will group legs together with a yellow background in tensor networks to denote such reshaping operations: 
\begin{align*}
    %
,
\end{align*}
this will be discussed in detail in Section~\ref{subsec:reshaping}.

\paragraph{Tensor Network Edges}\label{subsec:contraction} In tensor network diagram, legs are of two types: contracted legs~(those connecting two tensors) and un-contracted legs, also called free legs, with one dangling end~(i.e., a leg that is not connected to any other tensor). As mentioned above, un-contracted legs correspond to free indices:  the number of free legs indicates the order of a tensor: scalars have no free legs, vectors have one, matrices have two, and higher-order tensors have three or more.  Contracted legs represent contractions: 
tensors can be connected along legs of the same size, which represents a summation~(contraction) over the connected modes. 
We use the terms summation and contraction interchangeably. The most common contraction operation is \textit{matrix multiplication}. For two matrices $\Ab\in\RR^{d_1\times R}, \Bb\in\RR^{R\times d_2}$, their matrix product corresponds to a contraction between the second mode of $\Ab$ and the first mode of $\Bb$:

\begin{align}\label{eq:matrix-product}
(\Ab\Bb)_{ij} 
=

= \sum_{r=1}^R\Ab_{ir}\Bb_{rj},\ \ \ \text{for}\ \ i\in[d_1],j\in[d_2]\ .
\end{align}
Since this equality is true for all indices $i$ and $j$, we can simply write
\begin{align}\label{eq:matrix-product_noidx}
\Ab\Bb 
=
%
 \ \in \mathbb R^{d_1\times d_2}.
\end{align}

 In the above diagram, we can see that there is a connection between the legs of the same size $R$. This is consistent with matrix multiplication where there is a sum over indices with the same dimension. Finally, since the resulting diagram has two free legs, it represents a matrix. Here, the final object is a matrix which can be represented as  a single node, demonstrating how nodes can be merged in tensor networks:
$$
\Ab\Bb =\Mb \Leftrightarrow 
%
.
$$
Contraction between two matrices can directly be extended to contractions between a matrix and a vector, a tensor and a matrix, a tensor a matrix and a vector products, etc. Here are some examples of tensor networks with their corresponding mathematical expressions: 
\begin{align}\label{eq:matvec}
\Ab\in\RR^{d_1\times R},\vb\in\RR^{R} \ : \ \
%
 
= \sum_{r=1}^R\sum_{k=1}^{d_3}\Tt_{ijr}\Ab_{rk}\vb_{k}.
\end{align}
As we can see in all the diagrams above, edges between two nodes represent summations. They also indicate that two tensors share dimensions of the same size on the corresponding mode. 
As explained before, the number of free legs in the tensor network corresponds to the order of the tensor it represents. For the previous examples, we have

$$
%
 
\in \mathbb R^{d_1\times d_2}.
$$
 
\paragraph{From mathematical expressions to tensor networks, and back} Tensor networks are simple to work with because there are no strict rules for representing legs and nodes. 
In tensor network diagrams, nodes and edges can be positioned arbitrarily in the plane, only the graph structure matters. 
For example, when translating matrix multiplication into  a tensor network diagram, the key feature is to ensure the dimension of the connected legs are consistent, e.g., for $\Ab\in\RR^{d_1\times R}, \Bb\in\RR^{R\times d_2}$  all diagrams below illustrate the \emph{same} matrix multiplication:
\begin{align*} 
\Ab\Bb
=
.
\end{align*}
Note that because of the sizes of $\Ab\in\RR^{d_1\times R}$ and $\Bb\in\RR^{R\times d_2}$, there is only one way to do the contraction between the two; hence, there would be no ambiguity in the diagrams above even if we omitted the dimensions on the edges. 

Conversely, translating tensor network diagrams into mathematical formulations can lead to multiple interpretations. For example, the following matrix multiplication diagram
$
%
$
can be interpreted differently depending on what we choose the shapes of $\Ab$, $\Bb$ and the resulting matrix to be:
\begin{itemize}
    \item if $\Ab\in\RR^{d_1\times R}, \Bb\in\RR^{R\times d_2}$ and the result is of size $d_1 \times d_2$, then the diagram represents the product $\Ab\Bb$,
    \item if $\Ab\in\RR^{R\times d_1}, \Bb\in\RR^{R\times d_2}$ and the result is of size $d_1 \times d_2$, then the diagram represents the product $\Ab^\ts\Bb$,
    \item if $\Ab\in\RR^{R\times d_1}, \Bb\in\RR^{R\times d_2}$ and the result is of size $d_2 \times d_1$, then the diagram represents the product $\Bb^\ts\Ab$,
    \item ...
\end{itemize}
Therefore, one should be mindful when translating tensor network diagrams into mathematical formulations. But this is also what makes tensor networks very useful! When we are working with the diagrams, rather than mathematical expressions, we do not need to care about, e.g., transposes.

\paragraph{Contracting two legs of the same node. } As explained previously, any two legs of the same dimension can be connected to form an edge in a tensor network, even two legs corresponding to the same node. Consider for example a fifth-order tensor $\Tt\in\mathbb R^{d_1\times d_2 \times d_3 \times d_2 \times d_4}$. Since its 2nd and 4th modes have the same dimension, we can contract them together. Since the resulting tensor network has 3 dangling legs, it represents a third-order tensor: 
\begin{align*}
&
 \in \mathbb R^{d_1\times d_3 \times d_4}.
\end{align*}
When contracting two legs of an $N$-th order tensor, the resulting tensor is of order $N-2$. This is closely related to the notion of \emph{partial trace} that we will present later. 
\subsection{Inner Product, Outer Product, Trace and Norm}\label{sec:inner-products}
In this section, we use tensor networks to show how the classical notions of inner product, outer product, trace, and norm can be generalized to higher-order tensors.

\paragraph{Inner product} Similarly to the classical inner~(Euclidean) product between vectors, the inner product of two $N$-th order tensors $\St,\Tt\in\RR^{d_1\times\dots\times d_N}$ is the sum of the products of their entries:
$$ \langle \St,\Tt \rangle = \sum_{i_1=1}^{d_1}\dots\sum_{i_N=1}^{d_N}\St_{i_1\dots i_N}\Tt_{i_1\dots i_N}.$$ In a tensor network diagram, the summation over all dimensions is obtained by connecting all the legs of the two tensors. This results in a tensor network without free legs, representing a scalar. For example, for vectors $\ab\in\RR^{d\times 1}, \bb\in\RR^{d\times 1}$ we have
\begin{align}\label{eq:inner-product}
\langle\ab,\bb\rangle = \sum_{i=1}^d\ab_i\bb_i
=
    \begin{tikzpicture}[baseline=\baseline]
   \tikzset{tensor/.style = {minimum size = 0.5cm,shape = circle,thick,draw=black,inner sep = 0pt}, edge/.style = {thick,line width=.4mm},every loop/.style={}}
    \def\y{0}
    \def\x{0}
    \node[tensor,fill=red!30!white] (a) at (\x,\y){\scalebox{0.85}{$\ab$}};
    \node[tensor,fill=blue!30!white] (b) at (\x+1,\y){\scalebox{0.85}{$\bb$}};
    \draw[edge] (A) -- (B) node [midway,above]{\scalebox{0.7}{\dimcolor{$d$}}};
\end{tikzpicture},
\end{align}
and for two third-order tensors $\St,\Tt\in\RR^{d_1\times d_2\times d_3}$ we have
\begin{align*}
\langle\St,\Tt\rangle = \sum_{i_1=1}^{d_1}\sum_{i_2=1}^{d_2}\sum_{i_3=1}^{d_3}\St_{i_1i_2i_3}\Tt_{i_1i_2i_3} =
\begin{tikzpicture}[baseline=\baseline]
    \tikzset{tensor/.style = {minimum size = 0.5cm,shape = circle,thick,draw=black,inner sep = 0pt}, edge/.style = {thick,line width=.4mm},every loop/.style={}}
    \def\y{0}
    \def\x{0}
    \node[tensor,fill=green!50!red!50!white] (St) at (\x,\y){\scalebox{0.85}{$\St$}};
    \node[tensor,fill=green!50!red!50!white] (Tt) at (\x+1,\y){\scalebox{0.85}{$\Tt$}};
    \draw[edge] (St) -- (Tt);
    \draw[edge] (St) to [out=45,in=135] (Tt);
    \draw[edge] (St) to [out=-45,in=-135] (Tt);
    \node[draw=none] () at (\x+0.5,\y+0.5) {\scalebox{0.7}{\dimcolor{$d_1$}}};
    \node[draw=none] () at (\x+0.5,\y+0.15) {\scalebox{0.7}{\dimcolor{$d_2$}}};
    \node[draw=none] () at (\x+0.5,\y-0.5) {\scalebox{0.7}{\dimcolor{$d_3$}}};
\end{tikzpicture}  
.
\end{align*}
In this book, we will focus only on real-valued tensors, but the tensor network formalism can easily be adapted to complex-valued ones. 
For example, if the two tensors were complex-valued, one would have to take the conjugates of the entries of the second tensor, which could be indicated in the tensor network by having node $\Tt^*$~(component-wise conjugate) instead of $\Tt$. 
\paragraph{Outer product}\label{subsection:outer-product}
Recall that the outer product of two vectors $\ub\in\RR^{m},\vb\in\RR^n$ is the $m\times n$ rank-one matrix $\ub\vb^\top$.
This operation can be generalized to any number of tensors. For example, the outer product of $N$ vectors, $\ab_1\in\RR^{d_1},\cdots,\ab_N\in\RR^{d_N}$, is the tensor of order $N$ defined by
\begin{align*}
    (\ab_1 \outprod \ab_2 \outprod \cdots \outprod \ab_N)_{i_1,i_2,\cdots,i_n}
    = (\ab_1)_{i_1} (\ab_2)_{i_2}  \cdots (\ab_N)_{i_N}\ \text{for all } i_1 \in [d_1],\cdots, i_N\in[d_N].
\end{align*}
Such a tensor is called a \textbf{rank-one} tensor. %
The diagram below illustrates the outer product of $N$ vectors: 
$$
\ab_1\outprod\ab_2 \outprod\cdots\outprod \ab_N=
\begin{tikzpicture}[baseline=\baseline]
    \tikzset{tensor/.style = {minimum size = 0.5cm,shape = circle,thick,draw=black,inner sep = 0pt}, edge/.style   = {thick,line width=.4mm}}
    \def\y{0}
     \def\x{0}
     \node[tensor,fill=blue!10!white] (a1) at (\x,\y){\scalebox{0.85}{$\ab_1$}};
    \node[tensor,fill=blue!10!white](a2) at (\x+1,\y){\scalebox{0.85}{$\ab_2$}};
  
    \node[tensor,fill=blue!10!white](an) at (\x+3,\y){\scalebox{0.85}{$\ab_N$}};
    \draw[edge] (a1) -- (\x,\y-0.7) node [midway, left] {\scalebox{0.7}{\dimcolor{$d_1$}}};
    \draw[edge] (a2) -- (\x+1,\y-0.7) node [midway, left] {\scalebox{0.7}{\dimcolor{$d_2$}}};
    \draw[edge] (an) -- (\x+3,\y-0.7)  node [midway, left] {\scalebox{0.7}{\dimcolor{$d_N$}}};
    \node (eq) at (\x+2,\y) {$\cdots$};
\end{tikzpicture}
\in\RR^{d_1\times\cdots\times d_N}.
$$
As a special case, for $N=2$, we have  $\ab_1 \outprod \ab_2 = \ab_1\ab_2^\ts\in\RR^{d_1\times d_2}$. As we see in the tensor network diagram of the outer product, there are no shared edges, which indeed reflects that no contraction (summation) occurs in an outer product. 
The notion of an outer product can naturally be extended to higher-order tensors.
Let $\St\in\RR^{m_1\times\dots\times m_p}$ and $\Tt\in\RR^{n_1\times\dots\times n_q}$. The outer product of $\St$ and $\Tt$ is the tensor of order $p+q$ defined by 
$$ (\St\outprod\Tt)_{i_1,i_2,\cdots,i_p,j_1,j_2,\cdots,j_q} = \St_{i_1,i_2,\cdots,i_p} 
\Tt_{j_1,j_2,\cdots,j_q}. $$
Like for vectors, the tensor networks of the outer product is simply obtained by juxtaposing the corresponding nodes without adding any edge:
$$
\St\outprod\Tt = 
\in\RR^{m_1\times\dots\times m_p\times n_1\times\dots\times n_q}.
$$
More generally, for any arbitrary number of tensors with arbitrary orders, the tensor network diagram of their outer product is obtained by simply placing them next to each other, e.g., for  $\Tt\in\RR^{d_1\times d_2\times d_3 \times d_4}$, $\Ab\in\RR^{m\times n}$ and $\vb\in\RR^{p}$
$$
\Ab\outprod \Tt\outprod \vb
=
%
\in\RR^{m\times n \times d_1\times d_2\times d_3 \times d_4\times p},
$$
is defined element-wise by $(\Ab\outprod\Tt\outprod \vb)_{i_1i_2j_1j_2j_3j_4k} = \Ab_{i_1i_2}\Tt_{j_1\cdots j_4}\vb_{k}$, where $i_1\in[m],i_2\in[n],j_l\in[d_l]$ for $l\in[4]$ and $k\in[p]$.

In the proposition below we show a simple visual proof of the fact that inner products of outer products are products of inner products.
\begin{proposition}
    Let $\Xt,\Tt\in\RR^{d_1\times d_2\times\cdots\times d_N}$ with $\Xt = \ab_1\outprod\ab_2\outprod\cdots\outprod\ab_N$ and $\Tt = \tb_1\outprod\tb_2\outprod\cdots\outprod\tb_N$, then $\langle\Xt,\Tt\rangle = \prod_{n=1}^N\langle\ab_n,\tb_n\rangle$.
\begin{proof}
\begin{align*}
\langle\Xt,\Tt\rangle
&=
\left\langle

=
\prod_{n=1}^N\langle\ab_n,\tb_n\rangle.
\end{align*}
\end{proof}
\end{proposition}
\paragraph{Edge of dimension one = no edge}
Note that in tensor network diagrams, an edge of size one is equivalent to having no edge. 
For example, in the matrix multiplication, if the contraction edge is of size one, it is equivalent to the outer product of vectors:
\begin{align*}
    \Ab\in\RR^{p\times 1},\ \ \ \Bb\in\RR^{1\times q}, \ \ \ 
(\Ab\Bb)_{ij} 
&=
%
\\
&=
(\ab\outprod\bb)_{ij},
\end{align*}
where $\ab$ is the unique column of $\Ab$ and $\bb$ the unique row of $\Bb$.
\paragraph{Trace} The trace operation is a special case of tensor contraction applied on square matrices. In tensor networks, the trace is naturally represented by a loop~(self-edge)  connecting the two legs of the matrix:
\begin{align*}
\Ab\in\RR^{d \times d},\hspace{0.5cm}
\tr(\Ab) = \sum_{i=1}^d\Ab_{ii} = \hspace*{-0.5cm}
%
\hspace*{-0.5cm}.
\end{align*}
Since there are no free legs, it is consistent with the fact that the trace is a scalar. Tensor network diagrams offer a very simple proof of the invariance of the trace under cyclic permutation. Let's start with the trace of the product of two matrices $\Ab \in \RR^{d\times R}$ and $\Bb \in \RR^{R\times d}$.  As we explained before, when drawing the graph of a tensor network, only the graph structure is important. This allows us to see that
\begin{align*}
\tr(\Ab\Bb)= 
\hspace{-0.5cm}
%
\hspace{-0.4cm} =
\tr(\Bb\Ab),
\end{align*}
and similarly for three matrices $\Ab\in\RR^{m\times n}, \Bb\in\RR^{n\times p},  \Cb\in\RR^{p\times m}$, we have
\begin{align*}
\tr(\Ab\Bb\Cb) = 
\hspace{-0.7cm}
%
\hspace{-0.7cm} = \tr(\Cb\Ab\Bb).
\end{align*}
Note that the equalities between the tensor network diagrams are trivial: we just changed the nodes' positions without changing the underlying graph's structure. We thus simply interpreted the same diagram in two ways, leading to a less trivial equality between $\tr(\Ab\Bb\Cb)$ and $\tr(\Cb\Ab\Bb)$. Note, however, that we cannot interpret the tensor network diagram above as the trace of $\Ab\Cb\Bb$ since the dimensions do not match in this matrix product\footnote{And even if all the matrices were square, the equality $\tr(\Ab\Bb\Cb)=\tr(\Ab\Cb\Bb)$ is false and could not be deduced from the tensor networks' equality above. This is because in a tensor network, each leg corresponds to a specific mode of a tensor (which can be hidden in the tensor network for square matrices of hypercubic tensors....)}!

\paragraph{Partial traces} The \emph{partial trace} is a common operation used, in particular, in the context of modeling many-body quantum systems, where density matrices of a part of a system are obtained by performing a partial trace of the density matrix of the whole system. Without delving into the specifics of quantum systems, consider a square matrix  of dimension $d_1d_2\times d_1d_2$ which we interpret as a fourth-order tensor 
.
The partial trace of $\Tt$ over the subspace $\mathbb R^{d_1}$ is the $d_2\times d_2$ matrix obtained by ``tracing out'' the $d_1$ dimension by connecting the corresponding legs: %
. Similarly, the partial trace of $\Tt$ over the subspace $\mathbb R^{d_2}$ is the $d_1\times d_1$ matrix 
%
.

\paragraph{Tensor Frobenius Norm} The Frobenius norm of a tensor $\Tt\in\RR^{d_1\times\cdots\times d_N}$ is the square root of the sum of its
squared elements~(i.e., the square root of the inner product with itself). In tensor networks, the norm of a third-order tensor $\Tt\in\RR^{d_1\times d_2\times d_3}$ is represented as
\begin{align*}
\norm{\Tt}_F^2 = \langle\Tt,\Tt\rangle = 
\sum_{i_1=1}^{d_1}\sum_{i_2=1}^{d_2}\sum_{i_3=1}^{d_3}\Tt_{i_1i_2i_3}^2
=
%
 .
\end{align*}
It is a direct and natural generalization of the matrix Frobenius norm: 
\begin{align*}
 \norm{\Ab}_F^2 = \tr(\Ab\Ab^\ts) = \tr(\Ab^\ts\Ab) =  \hspace*{-0.5cm}
%
 \hspace*{-0.5cm}.
\end{align*}

\paragraph{Frobenius norm of outer products} Using tensor networks, it is almost trivial to show that the Frobenius norm of an outer product is equal to the product of the Frobenius norms. E.g., for $\Ab\in\RR^{m\times n}$ and $\Tt\in\RR^{d_1\times d_2\times d_3}$, we have 
$$ %
=\norm{\Ab}_F^2 \norm{\Tt}_F^2
.
\end{align*}

\paragraph{Example} We conclude this section with a last example showing how the tensor network graphical notation is very useful for representing complicated operations on tensors:
\begin{align}\label{eq:einsum}
%
.
\end{align}
\subsection{Copy Tensors~(Hyperedges)}\label{sec:copy-tensor}

Sometimes, one needs to represent a simultaneous contraction shared among more than two tensors. Consider, for example, the following contraction between three vectors, $\ab\in\RR^d, \bb\in\RR^d,$ and $\cbb\in\RR^d$ to obtain a scalar: $\sum_{i=1}^d\ab_i\bb_i\cbb_i$.
This operation can be depicted in tensor networks using a special tensor called \emph{copy tensor}. The copy tensor is equivalent to a Kronecker delta, i.e., a hyper-diagonal tensor with ones on the diagonal and zeros elsewhere, and is represented as a black dot in tensor network diagrams. E.g., the copy tensor  

is defined element-wise as
$$
 \left(%
\right)_{ijk} = \delta_{ijk} = \begin{cases}
      1 & \text{if } i=j=k\\
      0 & \text{otherwise}
    \end{cases}.
$$
Using the copy tensor, the 3-way contraction mentioned above can be represented as 
\begin{align*}
%
 = \sum_{i,j,k=1}^d\delta_{ijk}\ab_i\bb_j\cbb_k =
\sum_{i=1}^d\ab_i\bb_i\cbb_i .
\end{align*}
We now give a more formal definition of the copy tensor of arbitrary order $N$.
\begin{definition}
N-th order copy tensor
   %
is the tensor of shape $\underbrace{d\times d\dots\times d}_{N\text{ times}}$ defined by 
\begin{align*}
       %
= \sum_{i=1}^d\underbrace{\eb_i\outprod\eb_i\outprod\dots\outprod\eb_i}_{N\text{ times}},%
\end{align*}
where $\eb_1,\eb_2,\dots,\eb_d$ are the vectors of the canonical basis of $\RR^{d}$.
\end{definition}

The name copy tensor comes from the fact that one can "copy" the canonical vectors by contraction with the copy tensor. Indeed, let $\eb_i\in\RR^{d}$ be a canonical basis vector, then one 
can check that
$       
%
.
$
Here is a short proof of this fact:
\begin{align*}
    \left(
%
.
\end{align*}
It is important to observe that the "copy" property of the copy tensor only holds for a canonical basis vector. indeed, it is not true for an arbitrary vector: in general
$       
%
$.
\begin{proposition}\label{prop:copy.tensor} In the following, we list some useful properties of copy tensors.
\begin{enumerate}
    \item The copy tensor of order one is an all-ones vector:
$$
  %
.
$$\label{itm:one-copy-tensor}
\item The copy tensor of order 2 is the identity matrix. Indeed, we have
$$
    %
 = \left(\sum_{k=1}^n \eb_k\eb_k^\ts\right)_{i,j} = \sum_{k=1}^n (\eb_k)_i(\eb_k)_j = \sum_{k=1}^n \delta_{ik}\delta_{jk} = \delta_{ij} = \begin{cases}
      1 & \text{if } i=j\\
      0 & \text{otherwise},
    \end{cases}
    =
    \Ib_{ij}.
$$
We can thus omit the black node in the middle and represent it simply as a line: 
$$
    %
$$
\item 
Interestingly, when contracting any number of legs from one copy tensor to legs of another copy tensor, the resulting tensor network is itself a copy
tensor (obtained by merging the contracted legs and dots in the tensor network diagram). For example: 
\begin{align*}  
\sum_{l_1}\delta_{i_1j_1k_1l_1}\delta_{i_2j_2k_2l_1} 
&=
%
=
\sum_{i_1}\delta_{i_1j_1k_1l_1}\delta_{i_1j_2k_2l_2}.
\end{align*}
\item\label{prop:diagcopy}
For any vector $\vb\in\RR^n$, let $\diag(\vb)\in\RR^{n\times n}$ denote the diagonal matrix having the entries of $\vb$ on the diagonal, and for any matrix $\Ab\in\RR^{n\times n}$, let $\diag(\Ab)\in\RR^{n}$ denote the vector containing the diagonal entries of $\Ab\in\RR^{n\times n}$.
Then,
$$
    \diag(\vb) = 
     %
.
$$
In this manuscript, we will sometimes draw nodes with a diagonal line, 
$
%
 
$, 
to represent a diagonal matrix.
\begin{proof}
For the first claim, for any $i,j\in [n]$, we have
\begin{align*}
    %
\right)\\
&= 
\sum_{k}\delta_{ijk}\vb_k = 
\delta_{ij}\vb_i
= 
\begin{cases}
\vb_i & \text{ if } i=j\\
0 &\text{ otherwise}
\end{cases}
= \text{diag}(\vb)_{ij}.
\end{align*}
For the second claim, for any $i\in [n]$, we have
\begin{align*}
    %
\right)
=
\sum_{j,k}\delta_{ijk}\Ab_{kj}
=
\Ab_{ii} = \diag(\Ab)
_{i}.
\end{align*}
\end{proof}
\item\label{prop:diag}
Consequently it is easy to check that $\diag(\vb)\diag(\ub)$ =      
    %
.

\end{enumerate}
\end{proposition}

{
\compiletikztrue

From the previous Proposition, we see that the copy tensor can be used to \textit{filter}~(select, extract) matching indices. Here are some examples illustrating this idea:
$$
%
   \ =\ 
    \Tt_{iijj}.
$$
       
We can use this property of the copy tensor to represent Hadamard~(component-wise) products as tensor networks. Indeed, the outer product $\ub \outprod \vb$ of two vectors contains every product of one entry of $\ub$ with one of $\vb$. The copy tensor can be used to filter only products corresponding to matching components of $\ub$ and $\vb$:

$$
%
   \ =\ 
\ub_i \vb_i.
$$
Thus, the component wise product between $\ub$ and $\vb$ is represented by the tensor network %
.
    Similarly, given two matrices $\Ab,\Bb \in\Rbb^{m\times n}$ of the same size, we can build their Hadamard product using two copy tensors:

$$
%
    \ =\
    \Ab_{i,j}\Bb_{i,j}\ .
$$

The Hadamard product will be discussed further in Section~\ref{sec:product}, along with the Khatri-Rao product which will provide another illustration of the usefulness of the copy tensor.

}

\subsection{Matrix Factorization in Tensor Networks}
Tensor factorizations can be depicted in tensor network diagrams, just like any other operation. In this section, we will introduce graphical diagrams for matrix factorizations. More general factorizations of higher-order tensors will be covered in Chapter~\ref{chapter:tensor-decomposition}. In tensor networks, factorization means splitting a single node into multiple nodes. We will refer to its inverse operation as merging, i.e., combining multiple nodes into a single node. This process can be clearly illustrated in the following tensor network diagram:
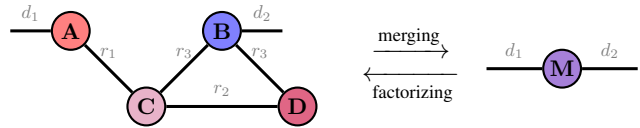
\begin{align*}
\begin{tikzpicture}[baseline=\baseline]
    \tikzset{tensor/.style = {minimum size = 0.5cm,shape = circle,thick,draw=black,inner sep = 0pt}, edge/.style = {thick,line width=.4mm},every loop/.style={}}
    \def\y{0}
    \def\x{0}
    \node[tensor,fill=red!50!white] (A) at (\x,\y){\scalebox{0.85}{$\Ab$}};
    \node[tensor,fill=blue!50!white] (B) at (\x+2,\y){\scalebox{0.85}{$\Bb$}};
    \node[tensor,fill=blue!30!red!30!white] (C) at (\x+1,\y-1){\scalebox{0.85}{$\Cb$}};
    \node[tensor,fill=red!80!blue!60!white] (D) at (\x+3,\y-1){\scalebox{0.85}{$\Db$}};
    \draw[edge](A) -- (C) node [midway,above] {\scalebox{0.7}{\dimcolor{$r_1$}}};
    \draw[edge](C) -- (D) node [midway,above] {\scalebox{0.7}{\dimcolor{$r_2$}}};
    \draw[edge](B) -- (D) node [midway,above] {\scalebox{0.7}{\dimcolor{$r_3$}}};
    \draw[edge](B) -- (C) node [midway,above] {\scalebox{0.7}{\dimcolor{$r_3$}}};
    \draw[edge](A) -- (\x-0.8,\y) node [midway,above] {\scalebox{0.7}{\dimcolor{$d_1$}}};
    \draw[edge](B) -- (\x+2.8,\y) node [midway,above] {\scalebox{0.7}{\dimcolor{$d_2$}}};
    \node[draw=none]() at (\x+4.5,\y-0.15){\scalebox{1}{$\xrightarrow[]{\text{merging}}$}};
    \node[draw=none]() at 
    (\x+4.5,\y-0.75){\scalebox{1}{$\xleftarrow[\text{factorizing}]{}$}};
    \node[tensor,fill=red!30!blue!50!white] (M) at (\x+6.5,\y-0.5){\scalebox{0.85}{$\Mb$}};
    \draw[edge](M) -- (\x+7.5,\y-0.5) node [midway,above] {\scalebox{0.7}{\dimcolor{$d_2$}}};
    \draw[edge](M) -- (\x+5.5,\y-0.5) node [midway,above] {\scalebox{0.7}{\dimcolor{$d_1$}}};
\end{tikzpicture}
\end{align*}
Therefore, we can represent any matrix decomposition in tensor network diagrams. We will see in the following chapters that the QR and singular value decompositions (SVD) are two decompositions that are often used in tensor network algorithm to split nodes. Both of these decompositions factorize a matrix into a product of matrices, some of them satisfying orthogonality constraints. We will use visual cues to represent such orthogonality constraints in tensor network diagrams:
\paragraph{Left and right orthonormal matrices and tensors}
We call a matrix $\Ub\in\RR^{m\times n}$ \emph{left-orthogonal} if the contraction  of its transpose with itself from  the left yields the identity matrix~($\Ub^\ts\Ub = \Ib_{n}$).
Similarly, a matrix $\Vb\in\RR^{m\times n}$ is \emph{right-orthogonal} if the contraction with its transpose from the right yields the identity matrix
($\Vb\Vb^\ts = \Ib_{m}$). In tensor network diagrams, we will use half-colored nodes to denote such orthogonality structure, in such a way that when connecting two copies of an orthogonal node using the legs on the uncolored part, the result is the identity. For example, a left orthogonal matrix $\Ub\in\RR^{m\times n}$ will be represented as 
 $\Ib_n$. 
Similarly, a right-orthogonal matrix $\Vb\in\RR^{m\times n}$ will be drawn as 
%
 and we have
%
 $\Ib_m$.

Note that if the matrix $\Ub$ is left orthogonal then $\Ub\Ub^\ts$ is not necessarily the identity, i.e., 
$$
%
\Ib_m,$$ this will only be the case if $m=n$.

Using these visual cues, QR decomposition and SVD will have the following tensor networks:
\begin{itemize}
    \item\textbf{QR Decomposition} The QR decomposition of $\Ab\in\RR^{d_1\times d_2}$ is given by 
    
$$
%
\in\RR^{d_1\times R}$
is the left-orthogonal matrix, i.e., $\Qb^\ts\Qb = \Ib_R$, $\Rb\in\RR^{R\times d_2}$ is upper triangular~(which is not apparent in the tensor network representation) and $R$ is the rank of $\Ab$. 
\item \textbf{Singular Value Decomposition~(SVD)} The~(compact) SVD of $\Ab\in\RR^{d_1\times d_2}$ is given by 
\begin{align*}
%
$ is diagonal and $R$ is the rank of $\Ab$. 
\end{itemize}
Assuming that $\Ab$ has the SVD represented in the previous diagram, we can give a short proof in tensor networks of the identity $\tr(\Ab^\ts\Ab) = \sum_{i}\sigma_i^2$, where $\sigma_i$ are the singular values of the matrix $\Ab\in\RR^
{m\times n}$:
$$
\tr(\Ab^\ts\Ab) 
%
= \tr(\Sigmab^2) = \sum_{i}\sigma_i^2,
$$
where we use the left and right orthogonal property of  $\Ub$ and $\Vb$.

\section{Operations on Tensors}\label{chapter:operations}
Recall that a tensor $\Tt \in \RR^{d_1 \times \cdots \times d_N}$ can be seen as a multi-dimensional array of order $N$: an $N$-th order or $N$-way tensor has $N$ modes, each mode corresponding to one dimension \citep{kolda2009tensor}.
\paragraph{Tensor Fibers} For any mode-$i$ (where $i = 1, \dots, N$), tensor \textit{fibers} are obtained by keeping all indices fixed except the $i$-th one. For example, a matrix column corresponds to a mode-1 fiber, while a matrix row corresponds to a mode-2 fiber. In a third-order tensor $\Tt \in \RR^{d_1 \times d_2 \times d_3}$, we have $d_2d_3$ mode-1 fibers, which are vectors of size $d_1$, i.e., $\Tt_{:,i_2,i_3} \in \RR^{d_1}$ for $i_2 \in [d_2]$ and $i_3 \in [d_3]$. The colon indicates varying the first index, while $i_2$ and $i_3$ remain fixed. Fibers are vectors.
\paragraph{Tensor Slices} Slices of a tensor are two-dimensional arrays (matrices) obtained by fixing all but two indices. For a third-order tensor $\Tt \in \RR^{d_1 \times d_2 \times d_3}$, there are horizontal, lateral, and frontal slices denoted by $\Tt_{i_1,:,:} \in \RR^{d_2 \times d_3}$, $\Tt_{:,i_2,:} \in \RR^{d_1 \times d_3}$, and $\Tt_{:,:,i_3} \in \RR^{d_1 \times d_2}$, respectively. Therefore, slices are matrices.
\paragraph{Scalar Multiplication} 
Let $\lambda \in \RR$ be a scalar. The scalar product $\lambda\Tt$ 
 is a tensor of the same dimension $d_1 \times \cdots \times d_N$, defined element-wise by:
$
    (\lambda\Tt)_{i_1 \cdots i_N} = \lambda \cdot \Tt_{i_1 \cdots i_N}$
for all admissible indices $i_k\in[d_k]$ and $k\in[N]$.
\paragraph{Tensors Sum} Let $\St\in\RR^{d_1\times\cdots\times d_N}$. The sum of tensors $\Tt$ and $\St$ is a tensor of the same dimension, $\Tt+\St\in\RR^{d_1\times\cdots\times d_N}$, defined element-wise by: $(\Tt+\St)_{i_1,\cdots,i_N} = \Tt_{i_1,\cdots,i_N} + \St_{i_1,\cdots,i_N}$ for all admissible indices $i_k\in[d_k]$ and $k\in[N]$.
\subsection{Permute and Reshape Tensors}\label{subsec:reshaping}
Permutation and reshaping are two fundamental operations on tensors.
\begin{itemize}
    \item \textbf{Permutation} rearranges the indices of a tensor without changing its number of modes. A common example of a permutation is the transpose of a matrix.
    \item  \textbf{Reshaping} combines multiple indices into larger ones or divides a large index into multiple smaller indices, reducing or increasing the total number of indices while preserving the overall size of the tensor.
\end{itemize}
We introduce matricization and vectorization, two common reshaping operations on tensors.
\begin{definition}(Matricization)\label{matricizitation:def}
Let $\Tt\in\RR^{d_1\times d_2\times \cdots\times d_N}$. The mode-$n$ matricization  $\Tt_{(n)}\in\RR^{d_n\times d_1d_2\cdots d_{n-1}d_{n+1}\cdots d_N}$, for $n\in[N]$, is obtained by unfolding $\Tt$ into a matrix by taking all mode-$n$ fibers and stacking them as columns. For example, the matricization of a tensor $\Tt\in\RR^{d_1\times d_2\times d_3}$ along mode-$1$, which is denoted by $\Tt_{(1)}$, is the $d_1\times d_2d_3$ matrix
$$
    \begin{bmatrix}
    \vline & \vline & \vline  & \vline \\
    \Tt_{:,1,1} & {\Tt}_{:,1,2}  & \dots  & {\Tt}_{:,d_2,d_3} \\
    \vline & \vline & \vline  & \vline 
    \end{bmatrix}
\in\RR^{d_1\times d_2d_3}.
$$

The mode-2 and mode-3 matricizations $\Tt_{(2)}\in\RR^{d_2\times d_1d_3}$ and $\Tt_{(3)}\in\RR^{d_3\times d_1d_2}$ are defined similarly.
\end{definition}

Observe that in mode-1 matricization, the two indices corresponding to the second and third modes of $\Tt$ are grouped together to form a new index that ranges from 1 to $d_2d_3$. In the previous definition, we chose to order this new index using lexicographic ordering, e.g., $11,12,13,21,22,23,31,32,33$ if $d_2=d_3=3$. Different orderings of the columns for matricizations are sometimes used~(e.g., reverse lexicographic in \citep{kolda2009tensor}). In general, the specific ordering
used is not important, but it must be consistent across related calculations~(see \citep{kolda2006multilinear} for further details).

In tensor network diagrams, we represent such a grouping of indices by merging the corresponding legs together:
$$\Tt_{(1)} = 
\left(

\in\RR^{d_1\times d_2d_3}.
$$
As we can see, the grouped legs highlighted with a yellow background are shaped like
$%
$
represents an edge of size $d_2d_3$ in tensor network diagrams.

In general, the notion of matricization can be extended to any subset $I\subset [N]$ of the modes of $\Tt$ which map  modes in $I$  to rows and the modes not in $I$ to columns, resulting in a matrix $\Tt_{(I)}$ of size $\prod_{i\in I} d_i \times \prod_{j\in [N]\setminus I} d_j$. For instance, for a sixth-order tensor $\Tt\in\RR^{d_1\times\dots\times d_6}$ we can group the indices as follows:
\begin{align*}
\Tt_{i_1,\dots,i_6} 
=
%
=
(\Tt_{(\{1,2,6\})})_{i_1i_2i_6 , i_3i_4i_5}\in\RR^{d_1d_2d_6\times d_3d_4d_5}.
\end{align*}
While this concept is intuitive, it is clunky to write formally: the indices $i_{r_1},i_{r_2},\cdots, i_{r_{|I|}}$ in $I$ (decreasingly ordered) of a tensor are mapped to the row index $
j = 1+\sum_{l=1}^{|I|}
\left[(i_{r_l}-1)\prod_{m=1}^{l-1}I_{r_m}\right]
$, while the remaining indices $i_{c_1},i_{c_2},\cdots, i_{c_{N-|I|}}$ (again decreasingly ordered) are mapped to the column index $k  = 1+\sum_{n=1}^{N - |I|}
\left[(i_{c_n}-1)\prod_{s=1}^{n-1}I_{r_s}\right]$. Using tensor networks allows us to abstract away this clunkiness by simply representing such 
reshaping operations by grouping legs together.

Likewise, just as a tensor can be converted into a matrix, it can also be converted into a vector. This process is called vectorization:
\begin{definition}(Vectorization)\label{def:vectorization}
Let $\Tt\in\RR^{d_1\times d_2\times\cdots\times d_N}$. The \emph{vectorization} of $\Tt$ denoted as $\vectorize(\Tt)\in\RR^{d_1d_2\cdots d_N}$ is the vector obtained by concatenating its mode-$N$ fibers\footnote{Using the lexicographic order.}. For example, the vectorization of a tensor $\Tt\in\RR^{d_1\times d_2\times d_3}$ is given by:
$$\vectorize(\Tt) =    
[(\Tt_{1,1,:})\ \  (\Tt_{1,2,:})\ \  \cdots \ \ (\Tt_{d_1,d_2,:})]^\ts
=
\begin{tikzpicture}[baseline=\baseline]
    \tikzset{tensor/.style = {minimum size = 0.5cm,shape = circle,thick,draw=black,inner sep = 0pt}, edge/.style = {thick,line width=.4mm},every loop/.style={}}
    \def\y{0}
    \def\x{0}
    \fill [reshapecolor]  (\x+1,\y+0.03) rectangle(\x+0.5,\y-0.03);
    \node[tensor,fill=blue!10!white] (A) at (\x,\y){\scalebox{0.85}{$\Tt$}};
    \draw[edge] (\x+0.15,\y+0.2) -- (\x+0.5,\y+0.2) -- (\x+0.5,\y+0.05) -- (\x+1,\y+0.05) node [midway,above]
    {\scalebox{0.7}{\dimcolor{$d_1$}}};
     \draw[edge] (A) -- (\x+1,\y)node [right]
    {\scalebox{0.7}{\dimcolor{$d_2$}}};
    \draw[edge] (\x+0.15,\y-0.2) -- (\x+0.5,\y-0.2)-- (\x+0.5,\y-0.05)-- (\x+1,\y-0.05) node [midway,below]
    {\scalebox{0.7}{\dimcolor{$d_3$}}};
    \end{tikzpicture}\in\RR^{d_1d_2d_3}.
$$
\end{definition}

Vectorization is a flattening operation that converts a tensor of any order into a vector. Similarly to matricization, in tensor networks
the right edge of the shape $
\begin{tikzpicture}[baseline=\baseline]
    \tikzset{tensor/.style = {minimum size = 0.5cm,shape = circle,thick,draw=black,inner sep = 0pt}, edge/.style = {thick,line width=.4mm},every loop/.style={}}
    \def\y{0}
    \def\x{0}
    \fill [reshapecolor]  (\x+1,\y+0.03) rectangle(\x+0.5,\y-0.03);
    \draw[edge] (\x+0.15,\y+0.2) -- (\x+0.5,\y+0.2) -- (\x+0.5,\y+0.05) -- (\x+1,\y+0.05) node [midway,above]
    {\scalebox{0.7}{\textcolor{gray}{$d_1$}}};
    \draw[edge] (\x+0.15,\y) -- (\x+1,\y)node [right]
    {\scalebox{0.7}{\textcolor{gray}{$d_2$}}};
    \draw[edge] (\x+0.15,\y-0.2) -- (\x+0.5,\y-0.2)-- (\x+0.5,\y-0.05)-- (\x+1,\y-0.05) 
node [midway,below]
    {\scalebox{0.7}{\dimcolor{$d_3$}}};
    \end{tikzpicture}
$ represents an edge of size $d_1d_2d_3$. In general, throughout this manuscript, convergent edges represent an edge whose size is the product of the sizes of all associated edges. Note that the vectorization of a tensor is equal to the vectorization\footnote{To be consistent with the definition of vectorization used in this book, the vectorization of a matrix is given by the concatenation of its rows, which is sometimes referred to as \emph{row-major vectorization}.Note that this is the default vectorizing operation, \textsf{ravel()}, of the \textsf{numpy} package in \textsf{Python}.} of the  mode-$1$ matricization of the tensor, i.e., for $\Tt\in\RR^{d_1\times d_2\cdots \times d_N}$, $\vectorize(\Tt) = \vectorize(\Tt_{(1)})$.

\subsection{Products}\label{sec:product}
Tensors can be combined or \textit{multiplied} using different types of operation. In this section, we provide graphical illustrations of different product operations.
\paragraph{Mode-$n$ Tensor Product}\label{par:mode-n-tensor-product}
The modes $n$ products, where a tensor is multiplied by a matrix along a specific mode, can be seen as a generalization of the matrix products. More precisely, the mode-$n$ product generalizes matrix-matrix multiplication along one particular mode. These include mode-$n$ products between tensors and matrices, as well as tensors and vectors.
\begin{enumerate}
\item \textbf{Mode-n product~(matrix).}\label{subsec:mode-n}
    The mode-$n$ product of a tensor $\Xt\in\RR^{d_1\times\cdots\times d_N}$ with a matrix $\Mb\in\RR^{m\times d_n}$ denoted as
    $\Xt\ttm n \Mb \in \RR^{d_1\times \cdots\times d_{n-1}\times m\times d_{n+1}\times\cdots\times d_N}$ is defined as the contraction of a tensor with a matrix along mode-$n$ of the tensor and mode-2 of the matrix~(column index), i.e., 
    $$
  \Xt\ttm n \Mb
=

\in\RR^{d_1\times  \dots\times d_{n-1}\times m \times d_{n+1}\times \dots\times d_N}.
$$
The operation contracts the tensor's $n$-th mode with the matrix's second mode~(column), replacing the original dimension $d_n$ with a new dimension $m$. For $\Ab\in\RR^{d_n\times n},\Bb\in\RR^{d_m\times m}$ and $\Xt\in\RR^{d_1\times\cdots\times d_N}$ with distinct modes $n\neq m$, the order of multiplication does not matter, i.e.,
$$
\Xt\ttm n \Ab \ttm m \Bb = 
%
= \Xt\ttm m\Bb\ttm n \Ab,~~\text{For}~~n\neq m.
$$
However, if $m = n$ the order of the products is important. In particular, if $\Ab\in\RR^{n\times d_n}$ and $\Bb\in\RR^{p\times n}$ we have $(\Xt\ttm n \Ab) \ttm n \Bb = \Xt\ttm n(\Bb\Ab)$, whereas the product $(\Xt\ttm n \Bb) \ttm n \Ab$ is not valid unless $n=d_n=p$.
The mode-$n$ tensor product can be seen as a generalization of classical matrix multiplication. In particular,
we can recover matrix products with mode-$1$ and mode-$2$ products of an order 2 tensor: 
if $\Ab\in\RR^{m\times n}$, $\Bb\in\RR^{p\times n}$ and $\Cb\in\RR^{d\times m}$, then
\begin{align*}
\Ab\ttm2 \Bb
=

    = \Cb\Ab\in\RR^{d\times n}.
\end{align*}

We show in the following proposition how the mode-$n$ product can be expressed through matricization and matrix product. 
\begin{proposition}
Let $\Xt\in\RR^{d_1\times\dots\times d_N}$ and $\Mb\in\RR^{m\times d_n}$ then 
$(\Xt\ttm n \Mb)_{(n)} = \Mb\Xt_{(n)}$.
\begin{proof}
We show this identity with a simple tensor network diagram for the special case $n=2$ and $N=3$. The extension to the general case is straightforward. We have
\begin{align*}
(\Xt\ttm2\Mb)_{(2)} 
&= 
\left(
    %
\\
&=
\Mb\Xt_{(2)} \in\RR^{m\times d_1d_3}.
\end{align*}
\end{proof}
\end{proposition}
\item\textbf{Mode-n product~(vector).} The mode-$n$ product of a tensor $\Xt\in\RR^{d_1\times\cdots\times d_N}$ with a vector $\vb\in\RR^{d_n}$, is denoted by 
$\Xt \ttm n \vb \in\RR^{d_1\times\cdots \times d_{n-1}\times d_{n+1}\times \cdots\times d_N}$. Note that the result is a tensor of order $N-1$. It can be pictured in a tensor network diagram as 
    $$
\Xt\ttm n \vb
=
%
\in\RR^{d_1\times  \dots\times d_{n-1}\times d_{n+1}\times \dots\times d_N}.
$$
Note that in mode-$n$ vector multiplication, unlike mode-$n$ matrix multiplication, the order of multiplication matters as a vector contraction reduces the tensor order by 1, so subsequent mode indices shift. Let $\Xt\in\RR^{d_1\times\cdots\times d_m\times\cdots\times d_n\times\cdots\times d_N}$ and $\ab\in\RR^{d_n},\bb\in\RR^{d_m}$, then
$
\Xt\ttm n \ab\ttm m\bb \neq
(\Xt \ttm m \bb)\ttm{n}\ab,
$
as mode-$n$ vector multiplication drops the $m$-th dimension~\citep{bader2006algorithm}.
\end{enumerate}

Next, we introduce several important products that are often used in tensor derivations: Kronecker, Khatri-Rao, and Hadamard products.

\paragraph{Kronecker product}\label{subsection-kron}
Let $\Ab\in\RR^{m\times n}$ and $\Bb\in\RR^{p\times q}$. The Kronecker product, $\Ab\kron\Bb\in\RR^{mp\times nq}$ is defined by
\begin{align}
    \Ab\kron\Bb = 
    \begin{bmatrix}
    a_{11}\Bb & a_{12}\Bb & \dots & a_{1n}\Bb \\
    a_{21}\Bb & a_{22}\Bb & \dots & a_{2n}\Bb \\
    \vdots & \vdots & \ddots & \vdots \\
    a_{m1}\Bb & a_{m2}\Bb & \dots & a_{mn}\Bb 
    \end{bmatrix}.
\end{align}

In tensor networks, the Kronecker product is given by the following diagram:
\begin{align}
    \Ab\kron\Bb = 
\begin{tikzpicture}[baseline=-0.5ex]
    \tikzset{tensor/.style = {minimum size = 0.5cm,shape = circle,thick,draw=black,inner sep = 0pt}, edge/.style   = {thick,line width=.4mm}}
    \def\x{0}
    \def\y{0}
    \node[tensor,fill=red!50!white] (A) at (\x,\y+0.3){\scalebox{0.85}{$\Ab$}};
    \node[tensor,fill=blue!50!white] (B) at (\x,\y-0.3){\scalebox{0.85}{$\Bb$}};
    \fill [reshapecolor] (\x-0.58,\y-0.05) rectangle (\x-1.2,\y+0.05);
    \draw[edge] (\x-0.25,\y+0.3) -- (\x-0.6,\y+0.3) -- (\x-0.6,\y+0.03)   -- (\x-1.2,\y+0.03) node [midway,above] {\scalebox{0.7}{\dimcolor{$m$}}};
    \draw[edge] (\x-0.25,\y-0.3) -- (\x-0.6,\y-0.3) -- (\x-0.6,\y-0.03) -- (\x-1.2,\y-0.03) node [midway,below] {\scalebox{0.7}{\dimcolor{$p$}}};
    \fill [reshapecolor] (\x+0.58,\y-0.05) rectangle (\x+1.2,\y+0.05);
    \draw[edge] (\x+0.25,\y+0.3) -- (\x+0.6,\y+0.3) -- (\x+0.6,\y+0.03) -- (\x+1.2,\y+0.03) node [midway,above] {\scalebox{0.7}{\dimcolor{$n$}}};
    \draw[edge] (\x+0.25,\y-0.3) -- (\x+0.6,\y-0.3) -- (\x+0.6,\y-0.03) -- (\x+1.2,\y-0.03) node [midway,below] {\scalebox{0.7}{\dimcolor{$q$}}};
\end{tikzpicture}.
\end{align}
As one can see from this diagram, the Kronecker product of two matrices is nothing but a reshaping of their outer products. Indeed, it is easy to convince oneself that every component $a_{i,j}b_{k,l}$ of the outer product $\Ab \outprod \Bb$ is present in $\Ab \kron \Bb$; while the outer product is a fourth-order tensor, the Kronecker product is one of its matricizations. One could show more precisely that $\Ab \kron \Bb = (\Ab \outprod \Bb)_{(\{1,3\})}$, but we will refrain from doing so and trust the simpler and direct view given by tensor network diagrams: grouping the row legs of $\Ab$ and $\Bb$ on one side, and their column legs on the other,  we obtain the matrix called the Kronecker product of $\Ab$ and $\Bb$:
$$\Ab\outprod \Bb =    

=\Ab \kron \Bb.
$$

More generally, the Kronecker product can be defined for any two tensors with the same order. It is obtained by grouping together and in order each pair of legs 
of the two tensors, e.g., 
$$
\St\kron\Tt = 
   %
    \in\RR^{m_1n_1\times\dots\times m_pn_p}.
$$

\begin{remark}
In the following, we list some useful properties of the Kronecker product.
\begin{enumerate}
    \item The Kronecker product is not commutative, i.e., $\Ab\kron\Bb \neq \Bb\kron\Ab$ in general%
    . 
    \item
    $
    \Ib\kron\Ab = 
    %
    $
    is a block diagonal matrix.
    \item The Kronecker product of two tensors of the same order results in a tensor of the same order, while their outer product produces a tensor with double the order.
    \item 
    By reshaping the Kronecker product $\Ab\kron\Bb$, the outer product $\Ab\outprod\Bb$ can be obtained.
 \item
The Kronecker product of two {identity matrices} is a larger identity matrix, i.e., $\Ib_m \kron \Ib_n = \Ib_{mn}$, which can easily be seen by drawing the tensor network diagram: 
$$
%
.
$$
     \item
     The Kronecker product of two vectors is the vectorization of their outer product: 
     $$\vectorize(\ub \outprod \vb) = \vectorize(\ub  \vb^\ts) =\ub \kron \vb =  %
    .$$
    \item 
   The Kronecker product satisfies the so-called \textit{mixed product} property: $(\Ab\kron \Bb)(\Cb\kron\Db) = \Ab\Cb\kron\Bb\Db$ where $\Ab\in\RR^{m\times n},\Bb\in\RR^{p\times q},\Cb\in\RR^{n\times d}$ and $\Db\in\RR^{q\times k}$.
\begin{proof}
 \begin{align*}
(\Ab\kron\Bb)(\Cb\kron\Db) 
&= 
%
\\
&=
\Ab\Cb\kron\Bb\Db.
    \end{align*}
\end{proof}
As a special case, we have $(\Ab\kron\Ib)(\Ib\kron\Bb)=\Ab\kron\Bb$. 
\item As we can see in the derivation above, it is often useful to reshape
$
    %
$ and vice versa when dealing with identities involving Kronecker products.
\item For $\Ab\in\RR^{m\times m}$ and $\Bb\in\RR^{n\times n},$ we have the equality $\tr(\Ab\kron\Bb)=\tr(\Ab)\tr(\Bb)$, which can easily be proved by the following diagram:
$$
\tr(\Ab\kron\Bb) = 
%
=
\tr(\Ab)\tr(\Bb).
$$

\item The \textbf{Sylvester Identity} is a very useful identity connecting vectorization, matrix product, and Kronecker product\footnote{\label{footnote:kronecker.order}Note that the order of the Kronecker products is different from the one appearing in~\cite{kolda2009tensor} since we use lexicographic ordering for matricizations and vectorizations~(see Definitions~\ref{def:vectorization} and~\ref{matricizitation:def}).} : $$\vectorize(\Ab\Xb\Bb) = (\Ab\kron \Bb^\ts)\vectorize(\Xb),$$
where $\Ab\in\RR^{m\times n},\Xb\in\RR^{n\times p}$ and $\Bb\in\RR^{p\times q}.$
\begin{proof}
The proof is rather direct using tensor networks: 
\begin{align*}
\vectorize(\Ab\Xb\Bb) &=
 \vectorize\left(

=
(\Ab\kron\Bb^\ts)\vectorize(\Xb).
\end{align*}
It is easy in the last diagram to identify the vectorization of $\Xb$ and the Kronecker product of $\Ab$ and $\Bb$, and thus to conclude that this tensor network represents a matrix-vector product between the two. However, when translating the tensor network in classical mathematical notation, some care is required. First, we can see that since $\vectorize(\Xb)$ is of dimension $np$, the Kronecker product has to be between $\Ab$ and $\Bb^\ts$ to obtain a matrix of shape $mq\times np$. Then, we need to know which Kronecker product to use: $\Ab \kron \Bb^\ts$ or $\Bb^\ts \kron \Ab$, which are both of the correct shape. Because in this manuscript we use lexicographic ordering for reshape operations such as $\vectorize$, the correct formulation is  $\Ab \kron \Bb^\ts$. However, if we were to use reverse lexicographic ordering, as in~\citep{kolda2009tensor}, the correct formulation would be $\Bb^\ts \kron \Ab$.
\end{proof}
The \emph{Sylvester identity} is a useful identity to solve systems of equations of the form $\Ab\Xb+\Xb\Bb = \Cb$.
\end{enumerate}
\end{remark}
There are useful relationships between the $n$-mode product, matricization, and Kronecker products: 
\begin{proposition}
Let $\Tt\in\RR^{d_1\times d_2\times\dots\times d_p}$, $\Ab_1\in\RR^{d_1\times n_1},
\Ab_2\in\RR^{d_2\times n_2},\dots,$
$\Ab_p\in\RR^{d_p\times n_p}$, then\footref{footnote:kronecker.order}
\begin{align*}(\Tt\ttm1\Ab_1\ttm2\Ab_2\ttm3\Ab_3\ttm4\dots\ttm p\Ab_p)_{(n)}=
\Ab_n\Tt_{(n)}(\Ab_1\kron\dots\kron\Ab_{n-1}\kron\Ab_{n+1}\kron\dots\kron\Ab_p)^\ts.
\end{align*}
\end{proposition}
\begin{proof}
We illustrate the proof for the case $p =3$ and $n=1$:
\begin{align*}
(\Tt\ttm1\Ab_1\ttm2\Ab_2\ttm3\Ab_3)_{(1)}
&= 
 \left(

=
\Ab_1\Tt_{(1)}(\Ab_2\kron\Ab_3)^\ts.
\end{align*}
\end{proof}
\begin{proposition}
    Let $\Tt\in\RR^{d_1\times d_2\times\cdots\times d_p}$  and $\Ab_i\in\RR^{d_i\times n_i}$ for $i\in[p]$. Let $m \in [p]$ and
$\Tt_{([m])}$ denote reshaping the tensor $\Tt$ in a matrix of size $d_1d_2\cdots d_m\times d_{m+1}...d_p$ by mapping the first $m$ indices to
rows and the remaining ones to columns. Then\footref{footnote:kronecker.order}
\begin{align*}
    (\Tt\ttm 1\Ab_1\ttm 2 \Ab_2\cdots\ttm p \Ab_p)_{([m])}=(\Ab_1\kron\Ab_2\kron\cdots\kron\Ab_{m})\Tb_{([m])}(\Ab_{m+1}\kron\Ab_{m+2}\kron\cdots\kron\Ab_{p})^\ts.
\end{align*}
This identity can be extended to arbitrary subsets $I\subset [p]$:
$$ (\Tt\ttm 1\Ab_1\ttm 2 \Ab_2\cdots\ttm p \Ab_p)_{(I)}    = \left(\bigotimes_{i\in I} \Ab_i\right)\Tt_{([m])}\left(\bigotimes_{i\in [p] \setminus I} \Ab_i\right)^\ts.$$
\begin{proof}
For $p=6$ and $I = \{1,2,3\}$,
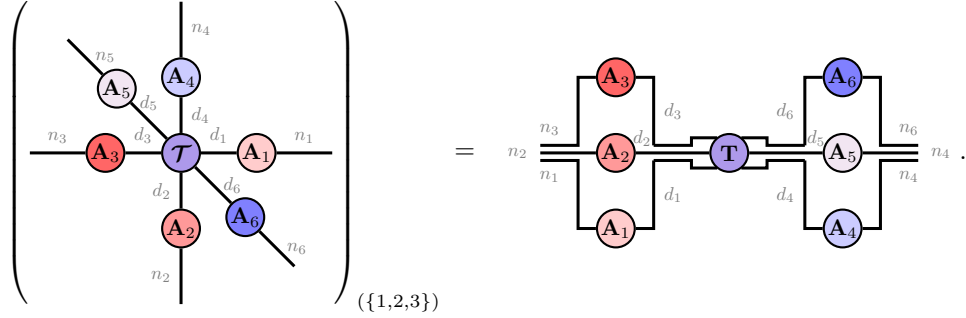
\begin{align*}
\left(
,
\end{align*}
\end{proof}
\end{proposition}
{
\compiletikztrue

\paragraph{Khatri-Rao product}
The Khatri-Rao product is the column-wise Kronecker product~\citep{smilde2005multi,bro1998multi}. Let $\Ab\in\RR^{m\times R}$ and $\Bb\in\RR^{n\times R}$. The Khatri-Rao product of $\Ab$ and $\Bb$, denoted as $\Ab\krao\Bb\in\RR^{mn\times R}$, is defined by
\begin{align*}
    \def\baseline{-0.5ex}
    \Ab\krao\Bb =
    %
\hspace{0.5cm},
\end{align*}
where $\ab_1,\dots,\ab_R\in\RR^{m}$ are the columns of $\Ab$ and $\bb_1,\dots,\bb_R\in\RR^{n}$ are the columns of $\Bb$.  
The Kathri-Rao product can be represented in tensor networks using the copy tensor. As discussed in Section~\ref{sec:copy-tensor}, the copy tensor can be used to filter matching indices, which is exactly what is needed to build the Kathri-Rao product:
$$
 %
    \ =\ 
    \Ab_{:,r} \kron \Bb_{:,r} = \ab_r \kron \bb_r.
$$
We thus see that the tensor network representation of the Khatri-Rao product is
$$
\Ab\krao\Bb \ =\  
 %
\ .
$$
In the tensor network diagram, the copy tensor enforces that the resulting matrix only contains the Kronecker product of \emph{matching} columns of $\Ab$ and $\Bb$.

}
\begin{remark}
We list useful properties of the Khatri-Rao product below:
\begin{enumerate}
\item Like the Kronecker product, the Khatri-Rao product is not commutative, i.e., $\Ab\krao\Bb\neq \Bb\krao\Ab$ in general.
\item The Khatri-Rao product is associative, i.e., $\Ab\krao(\Bb\krao\Cb)=(\Ab\krao\Bb)\krao\Cb$, where $\Ab\in\RR^{m\times R},\Bb\in\RR^{n\times R}$ and $\Cb\in\RR^{s\times R}$.
\item The columns of $\Ab\krao\Bb$ form a subset of the columns of the Kronecker product $\Ab\kron \Bb$.
\end{enumerate}
\end{remark}

\paragraph{Hadamard product}
Let $\Ab\in\RR^{m\times n}$ and $\Bb\in\RR^{m\times n}$ be matrices of the same shape. The Hadamard (or component-wise) product of $\Ab$ and $\Bb$ is the matrix  $\Ab\hadam\Bb\in\RR^{m\times n}$ defined element-wise by 
$$(\Ab\hadam\Bb)_{ij} = \Ab_{ij}\Bb_{ij}.$$
In tensor network diagrams, the Hadamard product can easily be represented using copy tensors:
\begin{equation*}
\Ab\hadam\Bb =
.
\end{equation*}
More generally, for any $\Ab_1\in\RR^{m\times n},\dots,\Ab_d\in\RR^{m\times n}$ we have 
\begin{align*}
\Ab_1\hadam\Ab_2\dots\hadam\Ab_{d} 
=
    %
.
\end{align*}
\begin{remark}
We list here some useful properties of the Hadamard product.
\begin{itemize}
    \item The Hadamard product of a tensor $\Tt$ with entries in $\{0,1\}$ with itself is $\Tt$. This is in particular true for the copy tensor:
    
    Let $\Tt\in\RR^{d_1\times d_2\times d_3}$ be a 3rd way copy tensor. Then the Hadamard product $\Tt\hadam\Tt$ in tensor networks is:
    $$\Tt\hadam\Tt=
    %
$$
    
    \item Using Proposition~\ref{prop:copy.tensor}~(items 4 and 5), we can see that 
    $\vb\hadam\ub=$
         %
    $=\text{diag}(\vb)\ub$
\end{itemize}
\end{remark}
To conclude this chapter, we introduce several useful results in the next section.

\subsection{Some Useful Properties of Tensor Networks}\label{sec:useful-TNs}
In this section, we list several useful results about tensor networks. 

\paragraph{Bounding the Rank of Matricizations.}
Any tensor network structure induces a collection of upper bounds on the matricizations of the underlying tensor. More precisely, for an arbitrary tensor network, the rank of a given matricization is bounded by the weight of any cut separating the row legs from the column legs in the graph. Here, the weight of a cut refers to the product of the dimensions of the edges that form the cut. This is formalized in the following theorem. 

\begin{theorem}(Rank of matricization of TN)\label{thm: matricization-rank}
Let $\T \in \Rbb^{d_1\times d_2\times \dots \times d_p}$ be given as a tensor network, and let $I \subset [p]$. Then, for any cut $\mathcal{C}$ separating the legs corresponding to indices in $I$ from the ones in $[p] \setminus I$ in the TN,
$$\rank(\T_{(I)}) \leq \prod_{e \in  \mathrm{cutset}(\mathcal C)} \mathrm{dim}(e),$$
where $\mathrm{cutset}(\mathcal{C})$ is the cut-set\footnote{Formally, the TN is a graph $G=(V,E)$. Each dimension of $\Tt$ is associated with one vertex in $V$: the one to which the corresponding dangling leg is attached. Any cut of this graph $\mathcal C$ is a partition of $V$: $\mathcal{C} = (V_1, V_2)$. In the theorem, we consider cuts that separate indices in $I$ from the other ones, i.e. such that $I \subseteq V_1$~(and thus $([p]\setminus I) \subseteq V_2$ since $(V_1,V_2)$ is a partition of $V$). The cut-set of $\mathcal{C}$ is the set of edges that have one endpoint in $V_1$ and one in $V_2$: $\mathrm{cutset}(\mathcal{C}) = \{(u,v) \in E \mid u\in V_1, v\in V_2 \}$. } %
of $\mathcal C$ and $\mathrm{dim}(e)$ is the dimension of the edge $e$ in the tensor network.
\end{theorem}
Let us illustrate this theorem with some examples. First, we can obtain the classical bound on the rank of a product of matrices as a particular case. Let   $\Ab\in\Rbb^{m\times R}$ and $\Bb\in\RR^{R\times n}$. For the matrix~(i.e., second-order tensor) $\mat{AB} \in \Rbb^{m \times n}$, the theorem gives us the bound

$$\rank(\Ab\Bb)=\rank
\left(
\right)\leq R.$$

Now consider the following tensor network representation of a $d_1\times d_2$ matrix:
$$
%
$$
where $\Tt\in\RR^{d_1\times R_1\times R_2\times R_3}$, $\St\in\RR^{R_1\times R_2\times R_3\times R_4}$, $\Bb\in\RR^{R_4\times R_5}$ and $\Mb\in\RR^{R_5\times d_2}$.
Theorem~\ref{thm: matricization-rank} can for example be used to obtain the upper bounds
$$
\rank\left(%
\right)\leq R_5.
$$
As a final example, let $\Gt\in\RR^{R_1\times\cdots\times R_4}$, $\Ab_1\in\RR^{R_1\times d_1}, \Ab_2\in\RR^{R_2\times d_2},\Ab_3\in\RR^{R_3\times d_3}$ and $\Ab_4\in\RR^{R_4\times d_4}$ and consider the tensor $\T = \Gt \times_1 \Ab_1\times_2 \Ab_2\times_3 \Ab_3\times_4 \Ab_4$. Theorem~\ref{thm: matricization-rank} can be used, for example, to obtain an upper bound on the rank of the first matricization of $\Tt$:
$$
\rank(\Tt_{(1)}) = 
\rank\left(
%
\right)\leq R_1.
$$
\paragraph{TN Edges and Sum Tensors.}Let $\T$ be a tensor given as a TN. Any edge in which the TN corresponds to one way to express the tensor as a sum of tensors of the same size as $\T$. This follows directly from the definition of edges in TNs. An edge represents a contraction, which is a sum. We illustrate this concept with examples. First, consider a rank-$R$ factorization of an $m\times n$ matrix. The edge in the TN corresponds to expressing the matrix as a sum of $R$ (rank-one) matrices:  
$$
%
= \sum_{r=1}^R \Ab_{:,r} \Bb_{r,:} = \sum_{r=1}^R \ab_r \outprod \bb_r,
$$
where $\ab_r$ are the columns of $\Ab\in\RR^{m\times R}$ and $\bb_r$ are the rows of $\Bb\in\RR^{R\times n}$.

Similarly, the unique edge in the TN representing the trace of a matrix corresponds to expressing the TN~(which is a scalar) as a sum of scalars~(the diagonal entries of the matrix), let $\Ab\in\RR^{d\times d}$:
$$%
\right)=\sum_{i=1}^d\Ab_{ii}.
    $$
Consider now a more interesting example:
$$%
.$$
The edge in this TN corresponds to a way to express the matrix represented by the TN as a sum of matrices:
$$%
    =
    \sum_{r=1}^R\Ab_r\kron\Bb_r.$$
We recognize here that this TN corresponds to a sum of Kronecker products! 
\paragraph{Tensors, Multilinear Maps and Polynomials.}
We can think of tensors as \emph{representations of multi-linear maps}, in exactly the same sense that matrices represent linear maps. Recall that a matrix $\Ab\in\RR^{m\times n}$ represents a linear map $\RR^n\to\RR^m$ via $\xb\mapsto \Ab\xb$. In tensor network notation, this is simply
$$
\yb = \Ab\xb
=
%
\in\RR^{m}.
$$
Linearity means that $\Ab(\alpha \xb+\beta \zb)=\alpha \Ab\xb+\beta \Ab\zb$ for all $\alpha,\beta\in\RR$ and $\xb,\zb\in\RR^n$. A tensor is the natural generalization when the map is linear \emph{in several arguments separately}. Concretely, an order-$N$ tensor $\Tt\in\RR^{d_1\times\cdots\times d_N}$ defines a map that takes $N$ input vectors and returns a scalar by contracting along all modes:
$$
f(\xb^{(1)},\dots,\xb^{(N)})
:= \langle \Tt,\ \xb^{(1)}\outprod \cdots \outprod \xb^{(N)}\rangle
=
\sum_{i_1=1}^{d_1}\cdots\sum_{i_N=1}^{d_N}\Tt_{i_1\cdots i_N}\,\xb^{(1)}_{i_1}\cdots \xb^{(N)}_{i_N}.
$$
This map is \emph{multi-linear}: it is linear in each $\xb^{(k)}$ when the other vectors are fixed. In tensor network diagrams, this is simply
\[
f(\ab,\bb,\dots,\zb)
=

\in\RR,
\]
where $\ab, \bb,\cdots,\zb\in\RR^d$.
If we leave one leg open, we obtain a \emph{vector-valued} multi-linear map. For instance, if $\Tt\in\RR^{d_1\times\cdots\times d_N \times p}$, contracting all modes except the last defines a multi-linear map whose output is a vector of dimension $p$:
$$
f(\ab,\bb,\dots,\zb)
=
%
\in\RR^p.
$$

\vspace{0.2cm}
\noindent
This immediately connects the tensors to \emph{multivariate polynomials}. The scalar map
\[
f(\xb^{(1)},\dots,\xb^{(N)})=\sum_{i_1,\dots,i_N}\Tt_{i_1\cdots i_N}\,\xb^{(1)}_{i_1}\cdots \xb^{(N)}_{i_N},
\]
is a polynomial in the entries of the vectors and is \emph{multi-linear}: every variable appears with degree at most one (when viewed in each block of variables $\xb^{(k)}$). In the important special case where all inputs are the \emph{same} vector $\xb\in\RR^{d}$ and $\Tt\in\RR^{d\times\cdots\times d}$ has $N$ modes, the  contraction
\[
\begin{tikzpicture}[baseline=2ex]
    \tikzset{tensor/.style = {minimum size = 0.5cm,shape = circle,thick,draw=black,inner sep = 0pt},
             edge/.style   = {thick,line width=.4mm},
             every loop/.style={}}
    \def\x{0}\def\y{0}

    \node[tensor,fill=green!50!red!50!white] (T) at (\x,\y+0.8){\scalebox{0.85}{$\Tt$}};

    \node[tensor,fill=blue!20!white] (a) at (\x-1.2,\y){\scalebox{0.85}{$\xb$}};
    \node[tensor,fill=blue!20!white] (b) at (\x-0.5,\y){\scalebox{0.85}{$\xb$}};
    \node[draw=none] (dots) at (\x+0.5,\y){$\cdots$};
    \node[tensor,fill=blue!20!white] (z) at (\x+1.2,\y){\scalebox{0.85}{$\xb$}};

    \draw[edge] (T) -- (a) node [midway,left]{\scalebox{0.7}{\dimcolor{$d\ $}}};
    \draw[edge] (T) -- (b) node [near end,right]{\scalebox{0.7}{\dimcolor{$d$}}};
    \draw[edge] (T) -- (z) node [midway,right]{\scalebox{0.7}{\dimcolor{$\ d$}}};
\end{tikzpicture}
= \Tt \ttm{1}\xb \ttm{2}\xb \cdots \ttm{N}\xb
=
\sum_{i_1,\dots,i_N}\Tt_{i_1\cdots i_N}\,\xb_{i_1}\cdots \xb_{i_N} \in \RR.
\]
is a \emph{homogeneous} multivariate polynomial of degree $N$ (every monomial has total degree exactly $N$). The tensor entries are precisely the coefficients of this polynomial in the monomial basis $(\xb_{i_1}\cdots \xb_{i_N})_{i_1,\cdots,i_N}$.

If instead we contract an order-$(N+1)$ tensor $\Tt\in\RR^{ d\times\cdots\times d \times p}$ with $N$ copies of $\xb\in\RR^d$, leaving the last leg free, we obtain a vector-valued homogeneous polynomial map $\RR^d\to\RR^m$:
\[
p(\xb) = \Tt \ttm{1}\xb \ttm{2}\xb\cdots \ttm{N}\xb \in\RR^{p}.
\]

\vspace{0.2cm}
\noindent
Finally, \emph{non-homogeneous} polynomials can be obtained using a simple trick: \emph{append a constant $1$ to the input vector}. Let $\xb\in\RR^d$ and define the augmented vector $\tilde{\xb} =\begin{bmatrix}\xb\\1\end{bmatrix}\in\RR^{d+1}$. If $\widetilde{\Tt}\in\RR^{(d+1)\times\cdots\times(d+1)}$ is an order-$N$ tensor, then
\[
\tilde{p}(\xb)
:=
\widetilde{\Tt} \ttm{1}\tilde{\xb}\ttm{2}\tilde{\xb}\cdots\ttm{N}\tilde{\xb},
\]
is a polynomial in $\xb$ of degree at most $N$ (some factors can pick the last coordinate, which is the constant $1$, thereby lowering the degree of the corresponding monomials). In other words, a single homogeneous tensor polynomial in the augmented variables $(x_1,\dots,x_d,1)$ compactly represents a general multivariate polynomial in $(x_1,\dots,x_d)$. This viewpoint is especially useful in practice: it lets us treat bias terms and lower-degree interactions using the same contraction machinery as the homogeneous case, simply by adding one extra dimension to each mode and fixing it to $1$.  

\section{Tensor Decompositions}\label{chapter:tensor-decomposition}
Working with high-order tensors is computationally expensive because the number of elements grows exponentially with the tensor's order. Tensor decompositions have emerged as powerful and efficient tools to address this issue.
Similar to matrix factorizations, tensor decompositions break down a high-order tensor into smaller components with lower-order and fewer entries, making computations more tractable. However, the notion of the rank of a matrix can be extended in various ways for tensors, each of them associated with a different factorization/decomposition model.
In this chapter, we introduce the most well-known tensor decomposition models.
\subsection{CANDECOMP/PARAFAC~(CP) Decomposition}\label{sec:cp-decomposition}
The CP decomposition~\citep{hitchcock1927expression} of an $N$-th order tensor $\Tt\in\RR^{d_1\times\cdots\times d_N}$  decomposes $\Tt$ as the sum of a finite number of rank-one tensors. Equivalently, it is a linear combination of $R$ rank-one tensors where $R$ is called the  rank of the decomposition: 
$$\Tt = \sum_{r=1}^R \ab^{(r)}_1\outprod\dots\outprod\ab^{(r)}_N ,$$
where $\ab^{(1)}_n,\dots,\ab^{(R)}_n\in\RR^{d_n}$ for each $n\in [N]$.

By grouping all the vectors $\ab^{(r)}_n$ in factor matrices,
\[\Ab_1 = {
}{\in\RR^{d_N\times R}}, \]
the CP decomposition is concisely noted as $\Tt = \CP{\Ab_1,\cdots,\Ab_N}$.

In tensor networks, a CP decomposition $\Tt = \CP{\Ab,\Bb,\Cb}$ is represented by
$$
%
,
$$
where the black dot is the third-order copy tensor introduced in Section~\ref{sec:copy-tensor}. Indeed, we have
\begin{align*}
    \left(%
\right)_{ijk} 
&= \sum_{r_1=1}^R\sum_{r_2=1}^R\sum_{r_3=1}^R\delta_{r_1r_2r_3}\Ab_{ir_1}\Bb_{jr_2}\Cb_{kr_3}
=
\sum_{r=1}^R\Ab_{ir}\Bb_{jr}\Cb_{kr}\\
&=
\sum_{r=1}^R(\ab_r)_i(\bb_r)_j(\cbb_r)_k 
=
\sum_{r=1}^R(\ab_r\outprod\bb_r\outprod\cbb_r)_{ijk}\\
&=
\left(\sum_{r=1}^R\ab_r\outprod\bb_r\outprod\cbb_r\right)_{ijk}
\end{align*}
\paragraph{CP Rank of a Tensor} 
The most fundamental and oldest concept of the tensor rank is the CP rank, first introduced by~\citep{hitchcock1927expression}. The CP rank of a tensor is defined as the minimum number of rank-one tensors required to express the tensor as their sum.\footnote{Note the subtle difference in our terminology between the rank of a CP decomposition and the CP rank of a tensor:  the rank of a CP decomposition is the number of summands in a decomposition of the form $\At = \sum_{r=1}^R \ab^{(r)}_1\outprod\dots\outprod\ab^{(r)}_N$, whereas the CP rank of a tensor $\At$ is the smallest $R$ for which such a decomposition exists. Consequently, if $\At$ has a rank-$R$ CP decomposition, its CP rank is at most $R$, but could be lower.}
Although this definition is similar to that of matrix rank, the properties of tensor rank differ significantly. From a computational point of view, a key difference is that, unlike matrix rank, there is no polynomial-time algorithm to determine a tensor’s CP rank. In fact, computing the rank of a tensor is an NP-hard problem~\citep{hillar2013most}.
There are several variants of the tensor rank, each associated with a specific tensor decomposition. As we progress through this chapter, we will introduce some of these different types of ranks.

\begin{remark}
\newcommand*{\horzbar}{\rule[.5ex]{2.5ex}{0.5pt}}
We list here some interesting properties of the CP rank and the CP decomposition. 
\begin{enumerate}
\item 
From Theorem~\ref{thm: matricization-rank} and Proposition~\ref{prop:copy.tensor}~(item 3), one can show that if a tensor $\Tt$ admits a rank-$R$ CP decomposition, then all its matricizations have rank upper bounded by $R$:
\begin{proposition}
    If $\Tt = \CP{\Ab_1,\cdots, \Ab_4}$ is a rank-$R$ CP decomposition of $\Tt$, then $\rank(\Tt_{(I)}) \leq R$ for any $I\subset [N]$.
\end{proposition}
The following tensor network illustrates this result for the matricization of a fourth-order tensor $\Tt = \CP{\Ab_1,\Ab_2,\Ab_3, \Ab_N}$ along its first two modes.
\begin{align*}
    \left(
.
\end{align*}
\item  The CP rank of a tensor $\Tt \in \RR^{d_1\times \cdots \times d_N}$ can easily be upper bounded as~\citep{kolda2009tensor}
$$\cprank(\Tt)\leq \min_{n\in [N]} \prod_{i\neq n} d_i.$$
\item For second-order tensors $\Ab\in\RR^{d_1\times d_2}$, the CP rank corresponds exactly to the classical notion of the rank of a matrix.
\item The CP decomposition $\Tt = \CP{\Ab,\Bb,\Cb}$ can be expressed using the Kronecker delta 
$$\Tt_{i,j,k} = \sum_{r_1,r_2,r_3}\delta_{r_1,r_2,r_3}\Ab_{i,r_1}\Bb_{j,r_2}\Cb_{k,r_3},$$
as well as with mode-$n$ products (see~\ref{par:mode-n-tensor-product}) 
$$\Tt = \It\ttm{1}\Ab\ttm{2}\Bb\ttm{3}\Cb$$
where $\It$ is the third-order copy tensor. 
\item The rank of the second-order tensors (matrices), over the fields $\RR$ and $\CC$ is the same. However, for higher-order tensors ($N$-th order tensors with $N\geq3$) the rank can vary depending on the decomposition field~\citep{kruskal1989rank}.
\item The number of parameters of a CP decomposition of an $N$-th order $d$-dimensional tensors ($d_1 = \dots = d_N = d$) is $NdR$ instead of $d^N$ for the full tensor. If $R$ is small, the number of parameters can be significantly reduced. 
\end{enumerate}
\end{remark}
The following proposition shows how the mode-$n$ matricization of the CP decomposition of a tensor can be expressed using the Khatri-Rao product.
\begin{proposition}
Let $\Tt\in\RR^{d_1\times d_2\times\dots\times d_N}$. If $\Tt=\CP{\Ab_1,\Ab_2,\dots,\Ab_N}$, then\footref{footnote:kronecker.order} 
$$\Tt_{(n)}=\Ab_n\left(\Ab_1\krao\dots\krao\Ab_{(n-1)}\krao\Ab_{(n+1)}\krao\dots\krao\Ab_N\right)^\ts.$$
\begin{proof}
For $N=3$ and $n=1$,
\begin{align*}
\Tt_{(1)}
&=
\left(

=
\Ab_1(\Ab_2\krao\Ab_3)^\ts.
\end{align*}
\end{proof}
\end{proposition}

\subsection{Tucker Decomposition}
The Tucker decomposition \cite{tucker1966some} factorizes a $N$-th order tensor
into smaller tensor and factor matrices. In this decomposition, the smaller
tensor is called the \emph{core tensor}. The Tucker
decomposition consists of $N$ mode-$n$ products (see, e.g., Section~\ref{sec:product})
between the core tensor and the factor matrices.

Let $\Tt \in \RR^{d_1 \times \dots \times d_N}$. A Tucker decomposition of $\Tt$
is given by
\[
    \Tt
    = \Gt \times_1 \Ub_1 \times_2 \Ub_2\times_3 \dots \times_N \Ub_N,
\]
where $\Gt \in \RR^{R_1 \times \dots \times R_N}$ is the core tensor and
$\Ub_i \in \RR^{d_i \times R_i}$ for $i \in [N]$ are the factor matrices.
The $N$-tuple $(R_1,\dots,R_N)$ is called the \emph{rank} of Tucker
decomposition.

The Tucker decomposition of a fourth-order tensor $\Tt\in\RR^{d_1\times\cdots \times d_4}$ can be illustrated using tensor networks notation as follows~\footnote{We consider the Tucker decomposition with orthogonal factor matrices, commonly referred to as the higher-order singular value decomposition (HOSVD).}.
$$
.
$$

It is not difficult to show that the factor matrices
$\Ub_i \in \RR^{d_i \times R_i}$ can always be chosen to be unitary (see
item~\ref{item:Tucker_orth} in Remark~\ref{rem:tuckers} below).

\paragraph{Multi-linear rank} The \emph{multi-linear (or Tucker) rank}\footnote{Note the distinction between the rank $(R_1,\dots,R_n)$ of a Tucker \emph{decomposition},
given by the size of the core tensor $\Gt$ appearing in
$\Tt = \Gt \times_1 \Ub_1 \times_2 \dots \times_N \Ub_N$, and the
(Tucker) multi-linear rank of a tensor $\Tt$, which is defined as the
minimal such tuple.} of a tensor $\Tt$ is the smallest tuple $(R_1,\dots,R_N)$
for which there exists a rank-$(R_1,\dots,R_N)$ decomposition of $\Tt$. Since there is only 
a partial order on tuples of integers, one may ask whether this notion of rank is well defined. 
The following proposition shows that it is. 

\begin{proposition}
Let $\Tt = \Gt \times_1 \Ub_1 \times_2 \Ub_2\times_3\dots \times_N \Ub_N$ and 
$\Tt = \tilde\Gt \times_1 \tilde\Ub_1 \times_2 \tilde\Ub_2\times_3 \dots \times_N \tilde\Ub_N$ be two 
Tucker decompositions of the same tensor $\Tt$ of rank-$(R_1,R_2,\dots, R_n)$ and 
$(\tilde R_1,\tilde R_2,\dots,\tilde R_n)$, respectively. Then $\Tt$ admits a Tucker decomposition of rank-$(\min(R_1,\tilde R_1), \min(R_2,\tilde R_2), \dots, \min(R_N,\tilde R_N))$. 
\end{proposition}
\begin{proof}
Suppose $N=3$. Applying the mode-$1$ unfolding and Theorem~\ref{thm: matricization-rank} to each of the two Tucker decompositions of $\Tt$ we obtain the following:
\begin{align*}
    \rank\left(
\right)\leq R_1,
\end{align*}
and at the same time
\begin{align*}
    \rank\left(%
\right)_{(1)}\leq\min(R_1,\tilde{R}_1)$. 
Applying the same reasoning for $n=2,3$ we conclude that $\Tt$ admits a Tucker decomposition of rank-$(\min(R_1,\tilde{R}_1),\min(R_2,\tilde{R}_2),\min(R_3,\tilde{R}_3))$.
\end{proof}

\paragraph{Computing The Tucker Decomposition.} In contrast with the CP decomposition, which is NP-hard to compute, finding an exact Tucker decomposition (and in particular the one with minimal multi-linear rank) can be achieved in polynomial time by successive SVDs. 

The basic idea of the Tucker decomposition algorithm is to find the $R_n$ leading left singular vectors
in mode-$n$, independently of the other modes. This algorithm, known as \textit{Higher-Order SVD (HOSVD)}~\citep{de2000multilinear}, is depicted below.  For simplicity, we illustrate this process for $N=4$:
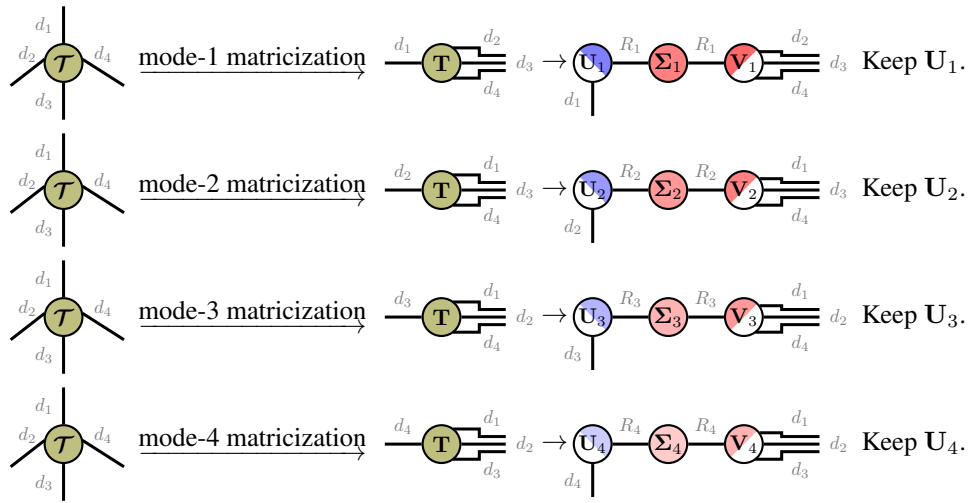
\begin{align*}
    &

\end{align*}
If all the SVD above are exact, i.e. $R_n = \rank(\Tt_{(n)})$ for each $n$, then we have that $\Tt \times_n \Ub_n\Ub_n^\ts = \Tt$ for all $n$, and thus $\Tt = \Tt\times_1\Ub_1\Ub_1^\ts\times_2\cdots\times_N\Ub_N\Ub_N^\ts$. As an example, for $N=4$ we have:
$$
%
.
$$
where the tensor $\Gt = \Tt\ttm{1}\Ub_1^\ts\ttm{2}\dots\ttm{4}\Ub_4^\ts$. 

The core tensor  $\Gt\in\RR^{R_1\times\dots\times R_4}$ along with the factor matrices $\Ub_1,\dots,\Ub_4$ forms a Tucker decomposition of the tensor $\Tt$. If all the SVDs of the matricization of $\Tt$ are exact, the resulting Tucker decomposition will be exact as well. When using truncated SVDs, the overall approximation error of the Tucker decomposition can be expressed as a function of the errors made in each truncated SVD~(see~\citep{de2000multilinear}).

\paragraph{Multi-linear rank and rank of matricizations}
The algorithm above shows that, if $\tilde R_n = \rank(\Tt_{(n)})$ for all $n$, then there exists a Tucker decomposition of rank $(\tilde R_1,\dots, \tilde R_N)$, showing that the multi-linear rank $( R_1,\dots, R_N)$ of $\Tt$ is upper bounded by the rank of the matricizations: $ R_n \leq \rank(\Tt_{(n)} )$ for all $n$. It is also easy to show that the multi-linear rank is lower bounded by the rank of the matricizations. Indeed, suppose $\T$ has multi-linear rank-$(R_1,\dots, R_N)$. Then, there exists a rank-$(R_1,\dots, R_N)$ Tucker decomposition $\Tt = \Gt\times_1\Ub_1\times_2 \dots \times_N \Ub_N$. Using Theorem~\ref{thm: matricization-rank}, we can easily see that this implies $\rank(\Tt_{(n)} )\leq  R_n$ for all $n$.
This proves the following theorem~\citep{de2000best}.
\begin{theorem}
The multi-linear rank-$(R_1,R_2,\dots, R_N)$ of a tensor $\Tt$ is given by the rank of its matricizations, i.e., $R_n = \rank(\Tt_{(n)})$ for $n\in[N]$.
\end{theorem}

\begin{remark}\label{rem:tuckers}
We list here some interesting properties of the Tucker decomposition and the Higher-Order SVD~(HOSVD) algorithm for computing it.
\begin{enumerate}

\item  If  $\Tt = \Gt\ttm{1}\Ub_1\ttm{2}\dots\ttm{N-1}\Ub_{N-1}\ttm{N}\Ub_N$ with all the $\Ub_i$ orthogonal, then $\norm{\Tt}_F = \norm{\Gt}_F$ and $\vectorize(\Gt) = (\Ub_1\kron\cdots\kron\Ub_N)\vectorize(\Tt)$. 
Indeed, for $N=3$ we have
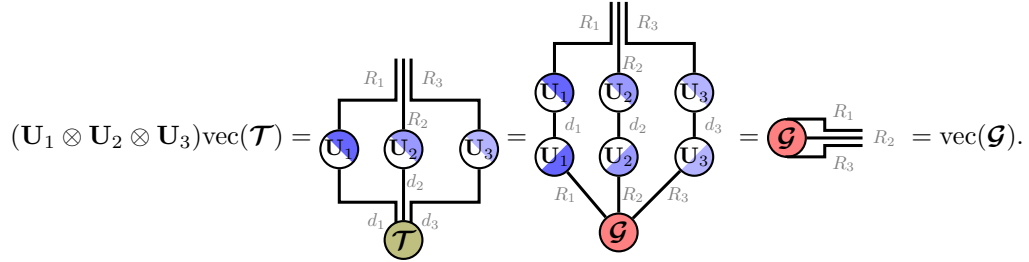
\begin{align*}
(\Ub_1\kron\Ub_2\kron\Ub_3)\vectorize(\Tt) &= 

    =\vectorize(\Gt),
\end{align*}
Observing that $\Ub_1\kron\Ub_2\kron\Ub_3$ is a left-orthogonal matrix, it follows that 
 $\norm{\Tt} = \norm{\vectorize(\Tt)} = \norm{(\Ub_1\kron\Ub_2\kron\Ub_3)\vectorize(\Tt)}= \norm{\vectorize(\Gt)} = \norm{\Gt}$.

    \item If we do not enforce any orthogonality constraints on the factor matrices and fix the core tensor to be a copy tensor,
$
%
,
$
we recover a CP decomposition. %
    \item \label{item:Tucker_orth} If $\Tt = \Gt\ttm{1}\Ab_1\ttm{2}\Ab_2\ttm{3}\Ab_3\ttm{4}\Ab_4$, with $\Ab_1,\Ab_2,\Ab_3$ and $\Ab_4$ not necessarily orthogonal, then there exists $\Tilde{\Gt}$ (of the same size as $\Gt$) and $\Qb_1,\Qb_2,\Qb_3$ and $\Qb_4$ orthogonal such that $\Tt = \Tilde{\Gt}\ttm{1}\Qb_1\ttm{2}\Qb_2\ttm{3}\Qb_3\ttm{4}\Qb_4$.
    \begin{proof}
By performing QR decomposition on each factor matrix, we obtain
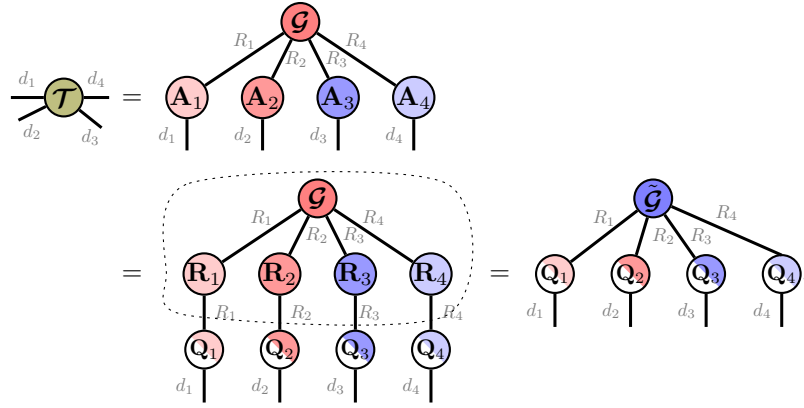
\begin{align*}
%
.
\end{align*}
\end{proof}
\item For an $N$-th order $d$-dimensional tensor, the number of parameters in its Tucker decomposition is $\bigo(R^N+NdR)$ assuming $R_1=\dots=R_N=R$.
\item
We conclude this section by showing how the mode-$n$ matricization of the Tucker decomposition of a tensor can be expressed using Kronecker products.
\begin{proposition}\label{decomp:tucker-kron}
Let $\At\in\RR^{d_1\times d_2\times\dots\times d_N}$. If $\At=\Gt\ttm{1}\Ub_1\ttm{2}\dots\ttm{N}\Ub_N$ then\footref{footnote:kronecker.order} 
$$\At_{(i)}=\Ub_i\Gt_{(i)}\left(\Ub_1\kron\dots\kron\Ub_{(i-1)}\kron\Ub_{(i+1)}\kron\dots\kron\Ub_N\right)^\ts.$$
\begin{proof}
Assume $N=3$ and $n=1$, then we obtain
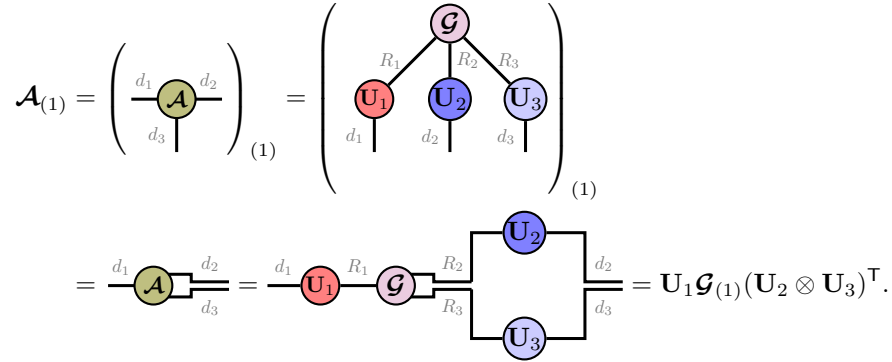
\begin{align*}
\At_{(1)}
&=
\left(

=
\Ub_1\Gt_{(1)}(\Ub_2\kron\Ub_3)^\ts,
\end{align*}
\end{proof}
\end{proposition}
\end{enumerate}
\end{remark}
\subsection{Tensor Train~(TT) Decomposition}\label{sec:tt-decomposition}
The Tensor Train~(TT) decomposition~\citep{oseledets2011tensor} is another popular tensor factorization model. It decomposes an $N$-th order tensor into $N$ smaller third-order tensors. 
A rank-$(R_1,\dots,R_{N-1})$ \textit{tensor train decomposition} of a tensor $\St\in\RR^{d_1\times\cdots\times d_N}$ factorizes it into a product of $N$  third-order tensors\footnote{The first and last cores can be though as matrices rather than third-order tensors since they have a singleton dimension. However, considering them as third-order tensors allows us to lighten some notations.}, $\Gt_n\in\RR^{R_{n-1}\times d_n\times R_n}$ for $n\in [N]$, with $R_0=R_N=1$:
$$\St_{i_1,\cdots,i_N}= \sum_{r_0=1}^{R_0}\sum_{r_1=1}^{R_1} \cdots \sum_{r_{N-1}=1}^{R_{N-1}}\sum_{r_N=1}^{R_N} \prod_{n=1}^N\Gt_n(r_{n-1},i_n,r_n) := \TT(({\Gt_n})_{n=1}^N),$$ 
for all $i_1\in[d_1],\cdots,i_N\in[d_N]$.  
The \emph{TT-rank} of $\St$ is the smallest
$(R_1,R_2,\dots,R_{N-1})$ such that $\St$ admits an exact rank-$(R_1,R_2,\dots,R_{N-1})$ TT-decomposition $\St=\TT(({\Gt_n})_{n=1}^N)$.
Similarly to the CP and Tucker decomposition, the rank can be viewed as a hyperparameter that controls the expressivity of the TT decomposition. That is, with a sufficiently large rank, a TT decomposition can represent any arbitrary tensor. The following tensor network illustrates the rank-$(R_1,R_2,R_3)$ TT decomposition of a fourth-order tensor: 

$$
\begin{tikzpicture}[baseline=-0.5ex]
    \tikzset{tensor/.style = {minimum size = 0.5cm,shape = circle,thick,draw=black,inner sep = 0pt}, edge/.style = {thick,line width=.4mm},every loop/.style={}}
    \def\x{6}
    \def\y{0}
     \node[tensor,fill=green!50!red!50!white] (B) at (\x+1,\y) {$\St$};
    \draw[edge] (B) -- (\x+0.3,0) node [midway,above] {\scalebox{0.7}{\textcolor{gray}{$d_1$}}};
    \draw[edge] (B) -- (\x+1.6,0) node [midway,above] {\scalebox{0.7}{\textcolor{gray}{$d_4$}}};;
    \draw[edge] (B) -- (\x+0.4,-0.3) node [midway, below] {\scalebox{0.7}{\textcolor{gray}{$d_2$}}};;
    \draw[edge] (B) -- (\x+1.5,-0.4)node [midway, below] {\scalebox{0.7}{\textcolor{gray}{$d_3$}}};;
\end{tikzpicture}=
\begin{tikzpicture}[baseline=\baseline]
    \tikzset{tensor/.style = {minimum size = 0.5cm,shape = circle,thick,draw=black,inner sep = 0pt}, edge/.style   = {thick,line width=.4mm},every loop/.style={}}
    \def\y{0}
    \def\x{0}
  \node[tensor,fill=red!30!white] (D) at (\x-1,\y){$\Gt_1$};
  \node[tensor,fill=red!50!white] (A) at (\x,\y){\scalebox{0.85}{$\Gt_2$}};
  \node[tensor,fill=blue!50!white] (B) at (\x+1,\y){$\Gt_3$};
   \node[tensor,fill=blue!20!white] (C) at (\x+2,\y){$\Gt_4$};
    \draw[edge](A) -- (\x,\y-0.7)node [midway,left] {\edgelabel{$d_2$}};
    \draw[edge](B) -- (\x+1,\y-0.7)node [midway,left] {\edgelabel{$d_3$}};
    \draw[edge](C) -- (\x+2,\y-0.7)node [midway,left] {\edgelabel{$d_4$}};
     \draw[edge](D) -- (\x-1,\y-0.7)node [midway,left] {\edgelabel{$d_1$}};
    \draw[edge](D) -- (A) node [midway,above] {\edgelabel{$R_1$}};
    \draw[edge](A) -- (B) node [midway,above] {\edgelabel{$R_2$}};
    \draw[edge](B) -- (C) node [midway,above] {\edgelabel{$R_3$}};
\end{tikzpicture}.
$$
The TT decomposition was initially introduced in the quantum physics and condensed matter community as \emph{Matrix Product States~(MPS)}. In this context, the dimensions of the internal edges are referred to as \textit{bond dimensions}, and the free legs are called \textit{physical dimensions}.
Next, we show how an exact TT decomposition of a tensor $\St\in\RR^{d_1\times\dots\times d_N}$ can be computed in polynomial time using successive singular value decomposition (SVD). This algorithm is known as the TT-SVD algorithm~\citep{oseledets2011tensor}, but was introduced earlier in the quantum physics community and referred to as \textit{successive Schmidt decompositions}~\citep{vidal2004efficient,mcculloch2007density}. 
Before presenting the TT-SVD algorithm, we introduce the notions of left and right orthogonality for third-order tensors.
\begin{definition}\label{def:left-right-ortho}A core tensor $\At_n\in\RR^{R_{n-1}\times d_n\times R_n}$ for $n\in[N]$ is left-orthogonal if $(\At_n)_{(3)}^\ts(\At_n)_{(3)}=\Ib_{R_n}$ and right-orthogonal if $(\At_n)_{(1)}(\At_n)_{(1)}^\ts=\Ib_{R_{n-1}}$. This can be illustrated using tensor networks 
as follows:
\begin{align*}
=\Ib_{R_{n-1}}.
\end{align*}
\end{definition}
The following theorem establishes the existence of a minimal tensor train decomposition for any tensor, and its proof is constructive, outlining the TT-SVD algorithm.
\begin{theorem}(\textbf{Computing (Orthogonal) Tensor Train Decomposition.})\label{thm:tt-decomposition}
For any $\St\in\RR^{d_1\times\dots\times d_N}$, let $\St_{([n])}\in\RR^{d_1\dots d_n\times d_{n+1}\dots d_N}$ be the matricization obtained by mapping the first $n$  modes of $\St$ to the rows of $\St_{[n]}$. Then the TT rank $(R_1,\cdots,R_{N-1})$ of $\St$  is given
by
    $R_n = \rank~(\St_{([n])})$ for all $n\in[N-1]$.
\end{theorem}
\begin{proof}
    Let $R_n = \rank~(\St_{([n])})$. We first show that we can construct a TT decomposition of $\St\in\RR^{d_1\times\cdots\times d_N}$ of rank $(R_1,R_2,\cdots,R_{N-1})$ by successive QR decomposition. For simplicity, we consider the case $N=4$, but the argument can easily be generalized.

We first take a QR decomposition of the first mode matricization of $\ten S$, which is of rank-$R_1$:
\begin{align*}

\end{align*}
Then, we perform a QR decomposition of $\mat R_1$ obtained in the previous step:
\begin{align*}
    %
\end{align*}\\
For the exact QR decomposition to be of rank-$R_2$, we first need to show that
the rank of $\Rb_1$, reshaped as a $R_1d_2 \times d_3d_4$ is at most $R_2$. Observe that, since $\rank(\St_{([2])}) = R_2$,
there exist $\Ab \in \RR^{d_1 d_2 \times R_2}$ and
$\Bb \in \RR^{R_2 \times d_3 d_4}$ such that $\St_{([2])} = \Ab\Bb$, i.e.,
$$
%
,
$$
By contracting $\Qb_1$ with the first-mode matricization $\St_{(1)}$ and using the above factorization, we obtain
\begin{align}\label{eq:tt-qr-fac}
    %
.
\end{align}
Moreover, since $\Qb_1$ is orthogonal and comes from the QR factorization of $\St_{(1)}$, we have that $\Qb_1\St_{(1)}$ is equal to $\Rb_1$, hence:
\begin{align}\label{eq:tt-equal-r1}
    %
.
\end{align}
By combining the identities~\eqref{eq:tt-qr-fac} and~\eqref{eq:tt-equal-r1}, we have
$$
%
.
$$
Using Theorem~\ref{thm: matricization-rank}, the TN above implies that reshaping $\Rb_1$ into a $R_1d_2\times d_3d_4$
matrix yields a matrix of rank at most $R_2$.
By repeating the steps above for $\Rb_2$, we ultimately obtain a TT decomposition with  rank-$(R_1,R_2,R_3)$ where $R_n = \rank{(\St_{([n])})}$ for $n\in [3]$, i.e.,
\begin{align}\label{eq:tt-format}
    %
.
    \end{align}
resulting in a TT decomposition where all cores are orthogonal except for the last core.

The algorithm above shows that, if $\tilde R_n = \rank(\St_{([n])})$ for all $n$, then there exists a TT decomposition of rank-$(\tilde R_1,\dots, \tilde R_{N-1})$, showing that the TT rank $(R_1,\dots, R_{N-1})$ of $\St$ is upper-bounded by the rank of the matricizations: $ R_n \leq \rank(\St_{([n])} )$ for all $n$. It is also easy to show that the TT rank is lower-bounded by the rank of the matricizations. Indeed, suppose that $\St$ has the TT rank  $(R_1,\dots, R_{N-1})$. Then, there exists a rank-$(R_1,\dots, R_{N-1})$ TT decomposition. Using Theorem~\ref{thm: matricization-rank}, we can easily see that this implies $\rank(\St_{([n])} )\leq  R_n$ for all $n$.
\end{proof}
Although the proof above uses successive QR decompositions to construct an exact TT decomposition, the algorithm commonly referred to as TT-SVD uses SVDs of the successive unfoldings. In the exact case, retaining all nonzero singular values yields TT ranks
$
R_n = \rank(\St_{([n])} )
$.
The advantage of the SVD formulation is that it immediately provides a low-rank approximation algorithm: at each step, truncated SVD can be used to prescribed TT ranks or according to a tolerance on the discarded singular values~\citep{oseledets2011tensor,holtz2012manifolds}.

Another widely used approach for decomposing a tensor into the tensor-train (TT) format is the Alternating Least Squares (ALS) algorithm. Before introducing this method, we first describe the canonical form of the TT decomposition.
\begin{definition}\label{def:canonical-form}
   The TT decomposition $\TT(({\At_n})_{n=1}^N)\in\RR^{d_1\times\dots\times d_N}$ is said to be in canonical form with respect to a fixed index $j\in [N]$ if all cores with $n<j$ are left-orthogonal and all cores with $n>j$ are right-orthogonal, e.g., the following TT decomposition is in canonical form with respect to its third core tensor:
$$
.
$$
The cores on the left side of $\At_3$ are left-orthogonal and the cores on the right are right-orthogonal. The core $\At_3$ is called the center of  orthogonality. Note that any TT decomposition can be efficiently converted to a canonical form w.r.t. any index $j\in [N]$ by performing a series of QR decompositions on the core tensors~\citep{holtz2012alternating,evenbly2018gauge}.
\end{definition}
\paragraph{Computing TT with Alternating Least Squares~(ALS)} The single-site Tensor Train Alternating Least Squares (TT-ALS) method~\citep{holtz2012alternating} starts with a TT decomposition in canonical form, initialized with a crude approximation~(e.g., with random cores, or using TT-SVD).
In the first step, the core 
$\At_1$ of the decomposition is non-orthogonal. During sweeps from left to right~(respectively, right to left), the algorithm keeps all cores fixed except for the $n$-th  one, $\At_n$, and optimizes it to minimize the distance to the target tensor. The resulting minimization problem is a simple least-squares problem, hence the name  ALS.  
After updating $\At_n$, a QR decomposition is applied to ensure numerical stability and to maintain the canonical form. The non-orthonormal factor is then merged into the left (or right, depending on the sweep direction), a step known as \textit{core orthogonalization}.
\begin{figure}[h!]
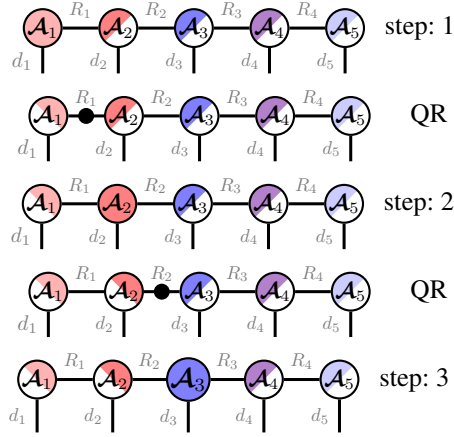

    \centering
~~~\makebox[1.2cm][l]{step:~3}
\caption{A half-sweep of TT-ALS}\label{fig:ortho-tt-als-alg}
\end{figure}
A half-sweep of TT-ALS is presented in Figure~\ref{fig:ortho-tt-als-alg}. In each non-QR step, the fully colored core is optimized, and in each QR step, the non-orthogonal component (depicted by a black circle)~is absorbed into the next core. This procedure repeats until reaching the right side of the decomposition, and then the same procedure is repeated from the right until reaching the left side (not demonstrated in this figure). The whole process is repeated until convergence. 
\begin{remark}
    This algorithm can be used with other loss functions in addition to the reconstruction error. In this case, the optimization of the loss w.r.t. one of the cores may not result in a least-squares problem, it is thus not called ALS anymore, but simply \emph{alternating minimization}. There also exists a two-site variant of this algorithm where pairs of consecutive cores are merged and optimized jointly at each iteration, before being split back in two using truncated SVD~\citep{holtz2012alternating}.
\end{remark}

In the next section, we first introduce the TT-matrix, also known as the Matrix Product Operator (MPO), and then demonstrate how to perform operations efficiently in the TT format.
\subsection{Efficient Operations in TT Format}\label{sec:efficient-operations}
A key advantage of the Tensor Train (TT) format is that it not only provides a low-rank representation of high-dimensional tensors, but also enables efficient execution of key algebraic operations.
In this section, we present the main operations that can be performed efficiently in TT format. We begin by describing the computation of inner products between two tensor-train (TT) tensors. We then introduce the Matrix Product Operator (MPO) representation for large matrices, followed by the matrix–vector product~(mat–vec) in the TT format. Next, we discuss TT rounding, a crucial procedure for controlling the growth of TT ranks. We conclude by outlining additional operations that can be performed efficiently within the TT framework, together with their computational complexities.
\paragraph{Inner Product}
As mentioned in Section~\ref{subsec:contraction}, for both tensors and vectors, the inner product is obtained by connecting all corresponding indices. Suppose that we have two fourth-order tensors $\Tt=\TT(({\Gt_n})_{n=1}^N)$, $\widetilde{\Tt}=\TT(({\Tilde{\Gt_n}})_{n=1}^N)\in\RR^{d_1\times...\times d_4}$ in the TT format. Then, the inner product can be computed efficiently, e.g.,
\begin{align}\label{eq:tt-inner}
\langle\Tt,\widetilde{\Tt}\rangle 
&=
,
\end{align}
where, for simplicity, we assume that all the ranks are equal to $R$.
The complexity of computing the inner product of two $N$-th order tensors in TT format with the process described in Eqn.~\eqref{eq:tt-inner} is $\bigo(NdR^4)$ provided that $d_1=\cdots=d_N=d$ and $R_1=\cdots=R_{N-1} = R$. This represents a significant improvement over the general form of the tensor inner product, where the two tensors are not represented in TT format, and the resulting computational complexity is $\bigo(d^N)$. Therefore, the TT format provides an efficient and practical approach for performing operations on high-dimensional tensors of low TT rank. Moreover, cores can be contracted in an even  more efficient manner:
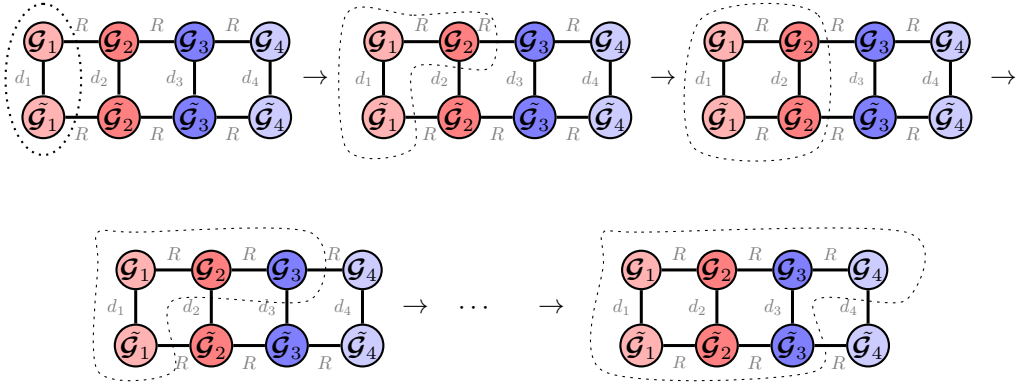
\begin{align*}

\end{align*}
As we can see above, the complexity of computing the inner product between two vectors of size $\bigo(d^N)$ can be reduced to $\bigo(NdR^3)$. In general, determining the optimal order of contraction for an arbitrary tensor network is an NP-hard problem~\citep{chi1997optimizing}.
\paragraph{Matrix Product Operator~(MPO)}\label{par:mpo}
As a generalization of the TT decomposition, we introduce the Matrix Product Operator (MPO)~\citep{oseledets2010approximation}. An MPO, also known as a TT-matrix, is a chain of four-way tensors that is used to represent a matrix. It was originally developed to describe operators acting on multi-body quantum systems. In simple terms, an MPO is a method of representing a matrix using tensors. Suppose that we have a matrix of size $\Mb\in\RR^{I_1I_2\dots I_N\times J_1J_2\dots J_N}$. For $n\in[N]$, let  $\At_n\in\RR^{R_{n-1}\times I_n\times J_n\times R_n}$  with $R_0=R_N=1$. A rank-$(R_1,\cdots,R_N)$ MPO decomposition of $\Mb$ is given by 
$$\Mb_{i_1i_2\dotsi_N,j_1j_2\dots j_N} = (\At_1)_{i_1,j_1,:}(\At_2)_{:,i_2,j_2,:}\dots(\At_{N-1})_{:,i_{N-1},j_{N-1},:}(\At_N)_{:,i_N,j_N},$$ 
for all indices $i_1\in[I_1],\cdots,i_N\in[I_N]$ and $j_1\in[J_1],\dots, j_N\in[J_N]$.
Here is an example of an MPO decomposition for a matrix $\Mb\in\RR^{I_1I_2I_3\times J_1J_2J_3}$ of rank~$(R_1,R_2)$ in the tensor network notation:
$$
.
$$
\paragraph{Matrix-vector Product in MPO} The product between a matrix 

$\Bb\in\RR^{I_1I_2I_3\times J_1J_2J_3}$
(given as a TT-matrix) and a vector $\vb\in\RR^{I_1I_2I_3}$
(given in TT format) can be efficiently computed:
\begin{align}\label{eq:tt-mat-vec}
%
.
\end{align}
In the last equality, each pair of cores $\At_i,\Gt_i$ has been merged to obtain the core $\Ht_i$~(for $i = 1,2,3$), resulting in a  TT vector with ranks $(R_1 S_1, R_2 S_2)$. More generally, for an $N$-th order tensor, if $I_1 = \cdots = I_N = I$, $J_1 = \cdots = J_N = J$, $R_1 = \cdots = R_{N-1} = R$, and $S_1 = \cdots = S_{N-1} = S$, the overall computational complexity of the matrix-vector product is $\mathcal{O}(N R^2 S^2 I J)$.

As illustrated with the matrix-vector product above, performing operations in the TT format can lead to an increase in TT ranks. This applies to basic linear algebra operations, such as summation and the entry-wise product of
tensors. Although the resulting tensors remain in the TT format, the ranks typically grow as a consequence of these operations~\citep{oseledets2011tensor}. A summary of these operations and their effects on the TT ranks is provided in Table~\ref{tab:tt-operations}~\citep{novikov2014putting}. 
\begin{table}
\centering
\begin{tabular}{lll}
\toprule
\textsc{Operation} & \textsc{Output Ranks} & \textsc{Complexity} \\
\midrule
$\Ct = \At \cdot \text{const}$ & $R(\Ct) = R(\At)$ & $\mathcal{O}(d\,R(\At))$ \\
$\Ct = \At + \text{const}$ & $R(\Ct) \leq R(\At) + 1$ & $\mathcal{O}(nd\,R(\At)^2)$ \\
$\Ct = \At + \Bt$ & $R(\Ct) \leq R(\At) + R(\Bt)$ & $\mathcal{O}(nd\,R(\Ct)^2)$ \\
$\Ct = \At \krao \Bt$ & $R(\Ct) \leq R(\At)\,R(\Bt)$ & $\mathcal{O}(nd\,R(\At)^2\,R(\Bt)^2)$ \\
$\mathrm{sum} (\At)$ & -- & $\mathcal{O}(nd\,R(\At)^2)$ \\
$\norm{\At}_F$ & -- & $\mathcal{O}(nd\,R(\mathbf{A})^3)$ \\
\bottomrule
\end{tabular}
\caption{ Efficient operations on tensors in the TT-format~(table taken from~\citep{novikov2014putting}.)}
\label{tab:tt-operations}
\end{table}
In this table, the term $R$ refers to the maximal TT rank, i.e., $R(\At) = \max_{i=0,\dots,N} R_i$, $n$ to the order of the tensor, and $d$ to its maximal dimension.

The TT-rounding algorithm~\citep{oseledets2011tensor} can be used to reduce the ranks when needed~(at the cost of some approximation error). 

\paragraph{TT-rounding} As noted earlier, tensor operations in the TT format, such as addition or multiplication, generally lead to an increase in the TT ranks. To prevent uncontrolled rank growth while maintaining accuracy, it is necessary to perform rank reduction. This is achieved by applying a \emph{TT-rounding} procedure: starting from an $N$-th order TT tensor with ranks $R^\prime_k$ for $k\in[N-1]$, the goal is to find a good approximation with lower TT rank $R_k$. A naive and inefficient way to achieve this would be to reconstruct the full tensor and perform TT-SVD to obtain a rank $(R_1,\cdots,R_{N-1})$ TT approximation. It turns out that this can be done without reconstructing the full tensor! We show that the following procedure is exactly equivalent to performing TT-SVD on the full tensor.

First, the tensor is orthogonalized by performing a sequence of QR decompositions on the TT cores in a left-to-right (or right-to-left) sweep. Subsequently, a truncation step is performed on each core using a truncated SVD, where singular values are discarded according to a prescribed rank-$R_k$. The truncated factors are then redistributed back into adjacent TT cores, yielding a new TT representation with reduced ranks while controlling the approximation error. As an illustrative example, let us consider a fourth-order tensor given in the TT format with rank-$R^\prime_k$. As a first step, we perform the QR decomposition from right to left:
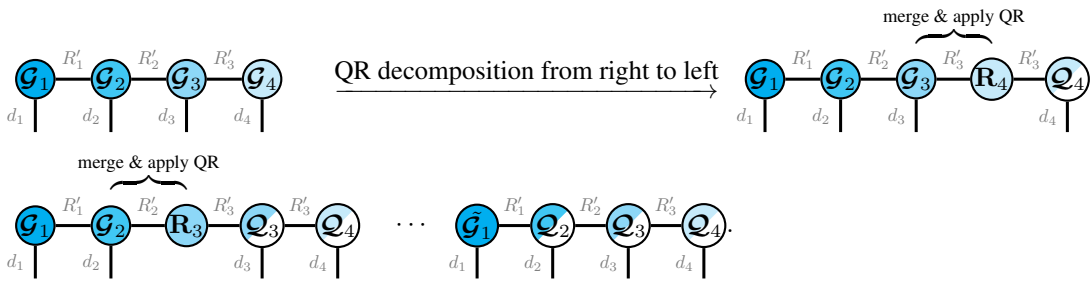
\begin{align*}
&
.
\end{align*}
The procedure continues with truncated SVDs from left to right:
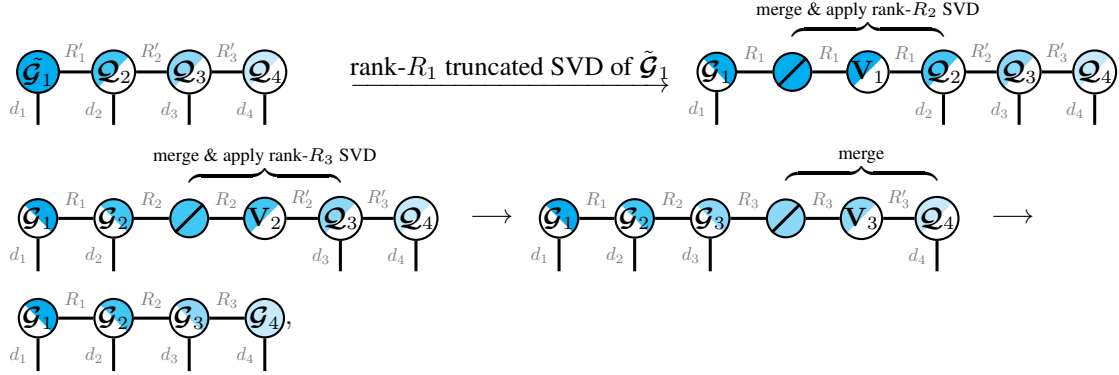
\begin{align*}
&%
,   
\end{align*}
where in the final step, all remaining tensors are contracted and combined into the final core $\Gt_4$. One can show that this procedure is \emph{exactly} equivalent to performing TT-SVD.

As expected, several other tensor networks are not covered in this chapter. In the next section, we briefly introduce some of them, which are notable generalizations of the Tucker and TT decompositions.
\subsection{Other decompositions: Tensor Ring, PEPS \& Hierarchical Tucker}

A first popular extension of Tensor Train is the Tensor Ring decomposition. 
\paragraph{Tensor Ring Decomposition}
The Tensor Ring (TR) decomposition is a generalization of the TT decomposition~\citep{zhao2016tensor}. Originally introduced in quantum physics, it has recently gained popularity in the machine learning community~\citep{wang2017efficient,wang2018wide,yuan2018higher}. While the TR decomposition is known to have some numerical stability issues~\citep{batselier2018trouble,chen2020tensor}, it generally requires fewer parameters and achieves better compression ratios than TT decomposition in practice~\citep{mickelin2020algorithms}. Let $\Xt\in\RR^{d_1\times\dots\times d_N}$
    be an $N$-th order tensor. Let 
    $\Gt_n\in\RR^{R_{n-1}\times d_n\times R_n}$ for $n\in[N]$ be the core tensors. 
    The rank-$R$ \emph{tensor ring decomposition} of the tensor $\Xt$ is given by
\begin{align*}
    \Xt_{i_1,\dots,i_N}=\hspace{0.25cm}\sum_{r_1=1}^{R_1}\dots\sum_{r_{N}=1}^{R_N}
    (\Gt_1)_{r_N,i_1,r_1} (\Gt_2)_{ r_1,i_2,r_2}\dots  (\Gt_{N-1})_{r_{N-2},i_{N-1},r_{N-1}} (\Gt_{N})_{r_{N-1},i_{N},r_N},
\end{align*}
for all indices $i_1\in[d_1],\cdots,i_N\in[d_N]$.
The TR decomposition can be represented in a tensor network notation for a fourth-order tensor as follows:
$$
.
$$
As a special case of tensor ring decomposition, we can also obtain a rank-one decomposition by setting all the ranks equal to one:
$$
%
=
\xb^1\outprod\xb^2\outprod\xb^3\outprod\xb^4.
$$
\paragraph{Projected Entangled Pair States (PEPS)}
The family of Projected Entangled Pair States~(PEPS) has emerged as one of the generalizations of the TT decomposition to higher-order lattice structures in quantum physics~\citep{verstraete2004renormalization,orus2014practical}.
PEPS are tensor networks structured as two-dimensional arrays of fifth-order core tensors. The following is the PEPS tensor diagram  corresponding to a $4\times 4$ square lattice:
$$
%
$$
\paragraph{Tree Tensor Networks/~Hierarchical Tucker Decomposition}
Another tensor network of particular interest is the Tree Tensor Networks~(TTN) or Hierarchical Tucker~(HT) decomposition, originally introduced in~\citep{hackbusch2009new, grasedyck2010hierarchical}. A key advantage of this tensor network is its higher expressive power compared to simpler models like the TT decomposition. The tensor network diagram for the HT decomposition of a fourth-order tensor is shown below:
$$
%
$$

\section{Computing Gradients of Tensor Networks}\label{chapter:gradients}
Optimizing tensor networks in a general setting is a significant challenge in many research areas. Although many optimization techniques have been developed for matrices, extending these methods to tensor networks can remain challenging~\citep{liao2019differentiable}. This complexity stems from the high computational cost of tensor contractions and the lack of efficient optimization algorithms for higher-dimensional cases. Additionally, manually computing gradients using the chain rule can seem daunting beyond simple tensor network structures. In this chapter, we introduce an elegant and intuitive approach for deriving higher-order derivatives using the graphical notation of a tensor network. This approach illustrates that derivatives of tensor network functions often inherit the low-rank structure of the initial tensor network.

In the next section, we first define the classical concepts of gradient and Jacobian for functions that take vectors as input. We then extend these concepts to functions that take tensors as inputs and produce tensors as outputs.
\subsection{Jacobians}
We first recall the classical notions of gradient and Jacobian of functions defined over $\mathbb{R}^n$.

\begin{definition}
  For $f:\RR^n\to\RR$ and $g:\RR^n\to\RR^p$. The gradient of $f$ and the Jacobian of $g$ at $\theta\in\mathbb{R}^n$ are defined by
    $$\nabla_\theta f = \left[\frac{\partial f(\theta)}
{\partial\theta_1},\frac{\partial f(\theta)}{\partial\theta_2},\dots,\frac{\partial f(\theta)} 
    {\partial\theta_n}\right]^\ts
        =
    \begin{tikzpicture}[baseline=\baseline]
    \tikzset{tensor/.style = {minimum size = 0.5cm,shape = circle,thick,draw=black,inner sep = 0pt}, edge/.style   = {thick,line width=.4mm}}
    \def\y{0}
     \def\x{0}
    \node[tensor,fill=blue!40!red!60!white] (a1) at (\x,\y){};
    \draw[edge](a1) -- (\x+0.75,\y)node [midway, above]{\edgelabel{$n$}};
    \end{tikzpicture} \in\mathbb{R}^n,$$
    and 
    $$\frac{\partial g(\theta)}{\partial\theta} = \left(\frac{\partial g(\theta)_i}{\partial\theta_j}\right)_{i=1,\cdots,p\atop j=1,\cdots, n}
    =
    \begin{tikzpicture}[baseline=\baseline]
    \tikzset{tensor/.style = {minimum size = 0.5cm,shape = circle,thick,draw=black,inner sep = 0pt}, edge/.style   = {thick,line width=.4mm}}
    \def\y{0}
     \def\x{0}
    \node[tensor,fill=blue!20!red!50!white] (A) at (\x,\y){};
    \draw[edge](A) -- (\x+0.75,\y)node [midway, above]{\scalebox{0.7}{\dimcolor{$n$}}};
    \draw[edge](A) -- (\x-0.75,\y)node [midway, above]{\scalebox{0.7}{\dimcolor{$p$}}};
    \end{tikzpicture}\in\mathbb{R}^{p\times n}.$$  
\end{definition}
We can naturally extend the notion of Jacobian to functions that take tensors as input and output tensors~(of arbitrary orders).

\begin{definition}
 Let $f:\RR^{n_1\times\dots\times n_N}\to\RR^{m_1\times\dots\times m_M}$,. The \emph{Jacobian tensor} of $f$ at $\theta\in\RR^{n_1\times\dots\times n_N}$ is the tensor $\frac{\partial f}{\partial\theta}\in\RR^{m_1\times\dots\times m_M\times n_1\times\dots\times n_N}$
 defined by
\begin{align*}
   \frac{\partial f(\theta)}{\partial\theta} 
   &= 
   \left(\frac{\partial f(\theta)_{i_1,\dots,i_M}}{\partial\theta_{j_1,\dots,j_N}}\right)_{i_1\in[m_1],\cdots,i_M\in[m_M]\atop j_1\in[n_1],\cdots,j_N
\in[n_N]}\\
&=
\begin{tikzpicture}[baseline=-0.5ex]
    \tikzset{tensor/.style = {minimum size = 0.5cm,shape = circle,thick,draw=black,inner sep = 0pt}, edge/.style = {thick,line width=.4mm},every loop/.style={}}
    \def\x{6}
    \node[tensor,fill=green!50!red!50!white] (T) at (\x,0) {};
    \draw[edge] (T) -- (\x-0.85,0.2)node [midway, above]{\scalebox{0.7}{\dimcolor{$m_1$}}};
    \draw[edge] (T) -- (\x-0.85,-0.2)node [left]{\scalebox{0.7}{\dimcolor{$m_2$}}};
    \draw[edge] (T) -- (\x-0.1,-0.75)node [near end, left]{\scalebox{0.7}{\dimcolor{$m_M$}}};
    \draw[dotted] (\x-0.55,-0.3) -- (\x-0.2,-0.45);
    \draw[edge] (T) -- (\x+0.85,0.2)node [midway, above]{\scalebox{0.7}{\dimcolor{$n_1$}}};
    \draw[edge] (T) -- (\x+0.85,-0.2)node [right]{\scalebox{0.7}{\dimcolor{$n_2$}}};
    \draw[edge] (T) -- (\x+0.1,-0.75)node [near end, right]{\scalebox{0.7}{\dimcolor{$n_N$}}};
    \draw[dotted] (\x+0.55,-0.3) -- (\x+0.2,-0.45);
    \end{tikzpicture}\in\RR^{m_1\times\cdots\times m_M\times n_1\times\cdots\times n_N}. 
\end{align*}
\end{definition}

The following theorem demonstrates that first-order derivatives of tensor networks with respect to a core tensor can be computed easily using tensor networks~\citep{roa2024probabilistic}.

\begin{theorem}\label{thm:gradients}
    Let $\Tt$ be a tensor given as a tensor network, where $\Gt$ is a core tensor \emph{appearing only once} in the tensor network. Then $\frac{\partial \Tt}{\partial \Gt}$ is obtained by removing $\Gt$ from the tensor network of $\Tt$. 
\end{theorem}
The following examples illustrate how Theorem~\ref{thm:gradients} can be applied to tensor networks.
\paragraph{Examples}
Let 
$

$ be a tensor network. Observe that we can think of $\Xt$ as a function of any of the core tensors of its tensor network representation. We can thus naturally ask what the Jacobian of $\Xt$ is with respect to, e.g., $\Gt_2$ or $\Gt_4$.
The previous theorem gives us an easy way to compute these derivatives:
\begin{itemize}
\item According to Theorem~\ref{thm:gradients}, to find the gradient of $\Xt$ with respect to $\Gt_2$, we need to remove the core $\Gt_2$ from the tensor network, i.e.,
$$
\frac{\partial}{\partial\Gt_2}\left(%
,
$$ 
where the uncontracted objects on the right-hand side of the above equality represent the outer product of the two objects (see, e.g.,~\ref{subsection:outer-product}).

To understand why this procedure works, we can view the tensor decomposition of $\Xt$ as a matrix-vector product:
$$
\Mb\vb 
=
%
,
$$
where $\vb = \vectorize(\Gt_2) $ and $\Mb$ is the matrix obtained by matricization of the remaining part of the tensor network. 
Using the classical identity for the Jacobian of a linear map, the expected result is obtained:
$$\frac{\partial\Mb\vb}{\partial\vb} = \Mb = 
%
\in\RR^{d_1d_2d_3\times R_1R_3R_4}.
$$
\item According to Theorem~\ref{thm:gradients} to find the gradient of $\Xt$ with respect to $\Gt_4$, we need to remove the core $\Gt_4$ from the corresponding tensor network. i.e., 
$$
    \frac{\partial}{\partial\Gt_4}\left(%
.
$$
Observe that the dangling leg of $\Gt_4$ remains in the Jacobian tensor! To understand why this is the case, we can again view the tensor decomposition of $\Xt$ as a matrix-vector product $\Mb\vb$, such that 
$$
\Mb\vb = %
,
$$
where now $\vb  =\vectorize(\Gt_4)$.
Taking derivative of this tensor network with respect to $\vb = \vectorize(\Gb_4)$~gives us the expected result:  $$\frac{\partial\Mb\vb}{\partial\vb} = \Mb
=
%
\in\RR^{d_1d_2d_3\times R_4}.$$
Note that taking the derivative with respect to a tensor removes only the corresponding node from the network, but leaves its connecting edges (legs) intact.
\end{itemize}
We now illustrate how Theorem~\ref{thm:gradients} can be used to effortlessly recover classical derivatives and gradient computations. 
\paragraph{Re-deriving Classical Identities with Tensor Networks}
\begin{enumerate}
\item 
    $\frac{\partial\langle\ub,\vb\rangle}{\partial \ub}
=
    \frac{\partial\left(%
=
    \Ib\outprod\xb
$,
\end{enumerate}
where in the last identity, the result is a third-order tensor.
Theorem~\ref{thm:gradients} applies when the tensor with respect to which the derivative is taken appears only once in the network. The following theorem shows how to handle the case where the tensor appears multiple times.
\begin{theorem}\label{thm:multiple-tensors-gradients}
    Let $\Tt$ be a tensor network where $\Gt$ is a core tensor. If $\Gt$ appears $k$ times in the tensor network of $\Tt$, then $\frac{\partial\Tt}{\partial\Gt}$ is obtained by summing $k$ copies of the tensor network of $\Tt$, where a different occurrence of $\Gt$ is removed in each copy.
\end{theorem}
The following examples illustrate how Theorem~\ref{thm:multiple-tensors-gradients} can be applied.
\paragraph{Examples}
\begin{itemize}
    \item Let $\xb\in\RR^{d}$ and $\Ab\in\RR^{d\times d}$
    \begin{align*}
    \frac{\partial \xb^\ts\Ab\xb}{\partial\xb}
=
    \frac{\partial\left(
\\
&=
2\Xb^\ts\Xb\Wb - 2\Xb^\ts\Yb.
\end{align*}
Since $\frac{\partial\tr(\Yb^\ts\Yb)}{\partial\Wb} = 0$, we omit it from the tensor networks diagram.
\end{itemize}
We conclude this chapter by presenting several applications of tensor network gradients.
\subsection{Applications}
There is a wide range of applications for tensor network gradients, we highlight two notable examples here. 

We first briefly introduce the notion of loss functions and explain how the chain rule can be easily applied with expressions involving tensors and Jacobian tensors. A loss function $L:S\times S\to \RR^{\geq 0}$ serves to measure the error or cost incurred by a computational model, algorithm, or decision process when producing the output $\hat{y}\in S$ as a prediction of the ground truth $y\in S$. The goal of many learning and optimization algorithms is to minimize this loss. One approach is to use the Frobenius norm for the difference between the predicted output and the ground truth. To minimize the Frobenius norm when both the input and output are tensors, Jacobians can be used to compute the necessary gradients. Let $\At\in\RR^{d_1\times d_2\times d_3}$ be the ground truth and $\Tt\in\RR^{d_1\times d_2\times d_3}$ be the predicted output. Using Theorem~\ref{thm:gradients}, the gradient of the loss $L = \norm{\At-\Tt}_F^2$ can be written as:
    \begin{align*}
    \frac{\partial\norm{\At-\Tt}_F^2}{\partial\Tt} 
    &=
    \frac{\partial}{\partial\Tt}\left(\norm{\At}_F^2 + \norm{\Tt}_F^2 - 2\langle\At,\Tt\rangle\right)\\
    &= 
    \frac{\partial}{\partial\Tt}\left(
\right).
\end{align*}
The chain rule can be easily extended to expressions involving tensor variables. Consider, for example, an arbitrary loss function $\mathcal{ L}: \mathbb{R}^{d_1\times \cdots \times d_N} \to \mathbb{R}$ that one wishes to minimize over the $N$-th order tensor inputs $\Tt$. Now suppose that the tensor input is parameterized as a tensor network where $\Gt \in \mathbb{R}^{m_1\times \cdots \times m_M}$ is an $M$-th order core tensor. To compute the Jacobian tensor of $\mathcal L$ with respect to $\Gt$, one can use the chain rule as follows: 
$$\frac{\partial  \mathcal L}{\partial\Gt} 
=
\frac{\partial \mathcal L}{\partial \Tt }\frac{\partial\Tt}{\partial\Gt}  $$
where $\frac{\partial  \mathcal L}{\partial\Gt}$ is of size $m_1\times\cdots \times m_M$, $\frac{\partial  \mathcal L}{\partial \Tt }$ is of size%
\footnote{One could say that,technically, $\frac{\partial  \mathcal L}{\partial\Gt}$ has shape $1\times m_1 \cdots \times m_M$  and $\frac{\partial  \mathcal L}{\partial \Tt }$ has shape $1\times d_1\times \cdots \times d_N$, but the first trivial mode of dimension $1$ is naturally omitted.}
$d_1\times \cdots \times d_N$, and $\frac{\partial\Tt}{\partial\Gt} $ is of size $d_1\times \cdots \times d_N \times m_1 \cdots \times m_M$. The product between $\frac{\partial  \mathcal L}{\partial \Tt }$ and $\frac{\partial\Tt}{\partial\Gt}  $ is implicitly understood as the contraction between the $M$ modes of $\frac{\partial  \mathcal L}{\partial \Tt }$ with the first $M$-modes  of $\frac{\partial\Tt}{\partial\Gt}  $.

We are now ready to conclude this chapter by presenting two applications of tensor network gradient computations. 
\begin{itemize}
\item\textbf{CP Decomposition with Gradient Descent} To find an approximate CP decomposition~(see Section~\ref{sec:cp-decomposition}) of a tensor $\Tt\in\RR^{d_1\times\cdots\times d_N}$, one can use gradient descent to minimize the loss $L = \norm{\Tt-\CP{\Gb_1,\cdots,\Gb_N}}^2_F$. Here $\CP{\Gb_1,\cdots,\Gb_N}$ denotes the CP decomposition where $\Gb_1,\cdots,\Gb_N$ are the factor matrices. To optimize w.r.t. the factor matrices, we can apply gradient descent to the loss function. For instance, in the case of $N=3$, to compute $\Gb_1$  we can apply the chain rule and write $\frac{\partial L}{\partial\Gb_1} = \frac{\partial L}{\partial\CP{\Gb_1,\Gb_2,\Gb_3}}\frac{\partial\CP{\Gb_1,\Gb_2,\Gb_3}}{\partial\Gb_1}$, where $\frac{\partial L}{\partial\CP{\Gb_1,\Gb_2,\Gb_3}}= 2(\Tt-\CP{\Gb_1,\Gb_2,\Gb_3})\in\RR^{d_1\times d_2\times d_3}$ and 
\begin{align*}
\frac{\partial\CP{\Gb_1,\Gb_2,\Gb_3}}{\partial\Gb_1}
=
\in\RR^{d_1\times d_2\times d_3\times R}.
\end{align*}
Looking at the resulting TN, one can see that it can be expressed using matricizations and the Khatri-Rao product:
\begin{align*}
  \frac{\partial L}{\partial\Gb_1} 
&=
\frac{\partial L}{\partial\CP{\Gb_1,\Gb_2,\Gb_3}}\frac{\partial\CP{\Gb_1,\Gb_2,\Gb_3}}{\partial\Gb_1}\\
&=
%
\\
&=
2(\Tt-\CP{\Gb_1,\Gb_2,\Gb_3})_{(1)}(\Gb_2\krao\Gb_3)
\in\RR^{d_1\times R}.
\end{align*}
 \item\textbf{Machine Learning} To compress neural networks, the dense weight matrices of fully connected layers can be parameterized using tensor factorizations to reduce memory usage~\citep{novikov2015tensorizing}. Let $\Wb\in\RR^{P\times D}$ be a weight matrix and $\xb\in\RR^{D}$ the input vector of a neural network layer. The output $\yb\in\RR^{P}$ is obtained by contracting the weight matrix $\Wb$ with the input vector $\xb$. 
    Suppose $P = \Pi_{i=1}^N{p_i}$ and $D =\Pi_{i=1}^Nd_i$. Then by reshaping $\Wb$ and $\xb$ in tensors we have
\begin{align*}
.
\end{align*}
For simplicity, we proceed with the case $N=4$. In~\citep{novikov2015tensorizing}, the authors propose to parameterize $\Wt$ in the MPO format~(see Section~\ref{sec:efficient-operations}). i.e., 
$\Wt = 
%
.$
During the backpropagation step, since we aim to update the weight matrix, we need to optimize it with respect to all the core tensors. According to the chain rule, this gives: $\frac{\partial L}{\partial \Gt_i} = \frac{\partial L}{\partial \Yt}\frac{\partial\Yt}{\partial\Gt_i}$ for $i\in[4]$, where $L$ denotes the loss function. For instance, for $i=2$ we need to find the derivative of $\Yt$ with respect to $\Gt_2$. By applying Theorem~\ref{thm:gradients} we just need to remove $\Gt_2$ from the tensor network to obtain the result:
\begin{align*}
\frac{\partial\Yt}{\partial\Gt_2} 
&=
\frac{\partial\left(
\in\RR^{p_1\times R_1\times p_2\times R_2\times p_3\times p_4\times d_2}.
\end{align*}
The gradient of the loss is thus given by
\begin{align*}
    \frac{\partial L}{\partial\Gt_2} = \frac{\partial L}{\partial\Yt}\frac{\partial \Yt}{\partial\Gt_2}
    =
    %
\in\RR^{R_1\times p_2\times R_2\times d_2}.
\end{align*}

\end{itemize}

\section{Probability Distributions and Random Tensor Networks}\label{chapter:prob}
In this chapter, we explore a very useful application of tensor networks: their ability to represent high-dimensional probability distributions in an intuitive and computationally efficient manner. For example, some tensor network decompositions, such as TT, allow efficient encoding of both probability distributions and their marginals~\citep{gardiner2024tensor, glasser2019expressive}. We start this chapter by showing how tensor networks can be used to efficiently model probability distributions. We then shift our focus to studying the statistical properties of tensors obtained by contracting random tensors together. Both parts of this chapter will illustrate how tensor networks can help us easily derive non-trivial results involving probability distributions and high-order tensors. 
\subsection{Probability Distributions and Tensor Decompositions}
In this section, we consider how probability distributions can be parameterized using tensor networks.
\subsubsection{Probability Distributions as Tensors}\label{sub:probs}
 One of the popular approaches to modeling a probability distribution with a tensor network is to represent the probabilities directly. Consider a multivariate probability mass function $\PP(X_1,\cdots,X_N)$ of $N$ discrete random variables. This joint probability distribution can be seen as an $N$-th order tensor with non-negative entries summing to one: 
\begin{definition}\label{def:probability}
Let $\PP$ be a joint distribution over $N$ discrete random variables $X_1,\cdots,X_N$ taking their values in $[d_1], [d_2], \cdots,[d_N]$, respectively. 
We say that the tensor $\Tt\in\RR^{d_1\times\cdots\times d_N}$ defined by $\Tt_{i_1\cdots i_N} = \PP(X_1=i_1,\cdots,X_N=i_N)$, \emph{represents} the probability $\PP$, which we denote by $\Tt \simeq  \PP(X_1,\cdots,X_N)$. By construction, the entries of $\Tt$ are non-negative and sum to one: 
$$
\sum_{i_1,\cdots,i_N}\Tt_{i_1,\cdots,i_N} 
=
\begin{tikzpicture}[baseline=\baseline]
    \tikzset{tensor/.style = {minimum size = 0.5cm,shape = circle,thick,draw=black,inner sep = 0pt}, edge/.style = {thick,line width=.4mm},every loop/.style={}}
    \def\x{0}
    \def\y{0}
    \node[tensor,fill=blue!20!white] (T) at (\x+1,\y) {$\Tt$};
    \draw  node[fill = black,circle,inner sep=0pt,minimum size=4pt] (m_1) at (\x,\y-0.65) {};
    \draw  node[fill = black,circle,inner sep=0pt,minimum size=4pt] (m_2) at (\x+0.5,\y-0.65) {};
    \draw  node[fill = black,circle,inner sep=0pt,minimum size=4pt] (m_3) at (\x+2,\y-0.65) {};
    \node[draw=none] at (\x+1.25,\y-0.65){$\dots$};
    \draw[edge] (T) -- (m_1)node [very near end,above]
    {\scalebox{0.7}{\dimcolor{$d_1$}}};
    \draw[edge] (T) -- (m_2)node [very near end,right]
    {\scalebox{0.7}{\dimcolor{$d_2$}}};
    \draw[edge] (T) -- (m_3)node [very near end,above]
    {\scalebox{0.7}{\dimcolor{$d_N$}}};
\end{tikzpicture}
= 1,
$$
where $
\begin{tikzpicture}[baseline=\baseline]
    \tikzset{tensor/.style = {minimum size = 0.5cm,shape = circle,thick,draw=black,inner sep = 0pt}, edge/.style = {thick,line width=.4mm},every loop/.style={}}
    \def\x{0}
    \def\y{0}
    \draw  node[fill = black,circle,inner sep=0pt,minimum size=4pt] (m_1) at (\x,\y) {};
    \draw[edge](m_1)--(\x+1,\y);
    \end{tikzpicture}
$
is a vector with all ones~(see item~\ref{itm:one-copy-tensor} in Proposition~\ref{prop:copy.tensor}).
\end{definition}
We now show how marginal and conditional probabilities can be expressed using tensor network notation.
\paragraph{Marginal Probabilities}
Let  $\Tt\in\RR^{d_1\times d_2\times d_3}$ be the tensor representing the probability distribution $\PP(X_1, X_2, X_3)$ then 
\begin{enumerate}
    \item Using the law of total probabilities, the marginal distribution of $X_2$ and $X_3$ can be expressed as
    
\begin{align*}
    \PP(X_2=i_2,X_3=i_3) 
    &=
    \sum_{i_1=1}^{d_1}\PP(X_1= i_1,X_2=i_2,X_3=i_3) \\
    &= 

 = \sum_{i_1=1}^{d_1}\Tt_{i_1,i_2,i_3},
\end{align*}
for all $i_2,i_3$. This shows that the tensor representing the marginal distribution can be obtained from a simple operation on the tensor representing the joint, i.e.,
\begin{align*}
    \text{If}~~~ %
\simeq\PP(X_2, X_3).
\end{align*}
\item
We can similarly express the marginal distribution of $X_2$:
\begin{align*}
\PP(X_2=i_2) 
&= \sum_{i_1,i_3=1}^{d_1,d_3}\PP(X_1=i_1,X_2= i_2, X_3=i_3)\\
&=
%
=\sum_{i_1,i_3=1}^{d_1,d_3}\Tt_{i_1,i_2,i_3},
\end{align*}
for all $i_2$. Hence, the marginal distribution for $X_2$ in the tensor network notation can be represented as:
\begin{align*}
    \text{If}~~~ %
\simeq\PP(X_2).
\end{align*}
\end{enumerate}
In the following, we use $\PP(i_1,\cdots,i_N)$ for simplicity to denote the probability $\PP(X_1=i_1,\cdots,X_N=i_N)$. 
As we can see, to represent marginals in tensor network notation, we simply need to contract a vector of all ones with the corresponding legs.
Similarly, we can represent conditional probability distributions.
\paragraph{Conditional Probabilities}
Let 
$\PP(i_n \mid i_1, \dots, i_{n-1})
$
denote the conditional probability of $X_n = i_n$ given that the preceding $n-1$ random variables take the values $i_1, \dots, i_{n-1}$.  
For a third-order probability tensor $\Tt\in \RR^{d_1 \times d_2 \times d_3}$, these conditional probabilities can be represented compactly using tensor network notation:
\begin{enumerate}
    \item 
We have $
\PP(i_1|i_2) = \frac{\PP(i_1,i_2)}{\PP(i_2)} = \frac{\sum_{i_3}\PP(i_1,i_2,i_3)}{\sum_{i_1,i_3}\PP(i_1,i_2,i_3)} = 
\frac{
},
$
for all $i_1, i_2$. 

Hence, for a fixed $i_2$, the conditional distribution $\PP(X_1|X_2=i_2)\propto%
},
$
for all $i_1, i_2$ and $i_3$.

Hence, for fixed $i_1$ and $ i_2$, the conditional distribution $\PP(X_3|X_1=i_1, X_2 = i_2)\propto%
.$ 
\end{enumerate}
\subsubsection{Tensor Train Parameterization of Probability Distributions}
As we saw in Section~\ref{chapter:tensor-decomposition}, working with high-order tensors becomes computationally expensive because the number of elements grows exponentially with the tensor order. This also applies to tensors representing joint probability distributions: the tensor representation grows exponentially with the number of random variables. One way to address this issue is to parameterize the probability tensor as a low-rank tensor network.

We will show here how this can be done with the tensor train decomposition. This can also be done with other decomposition models, but the tensor train format is particularly interesting and has been used very successfully in quantum physics to model many-body systems~\citep{TensorNetworkOrg}.

Let $\Tt\in\RR^{d_1\times\cdots\times d_N}$ be a probability tensor parameterized in the tensor train format: 
\begin{center}
\begin{tikzpicture}[baseline=-0.5ex]
    \tikzset{tensor/.style = {minimum size = 0.5cm,shape = circle,thick,draw=black,inner sep = 0pt}, edge/.style = {thick,line width=.4mm},every loop/.style={}}
    \def\x{6}
    \def\y{0}
     \node[tensor,fill=green!50!red!50!white] (B) at (\x+1,\y) {$\Tt$};
    \draw[edge] (B) -- (\x+0.3,0) node [midway,above] {\scalebox{0.7}{\dimcolor{$d_1$}}};
    \draw[edge] (B) -- (\x+1.6,0) node [midway,above] {\scalebox{0.7}{\dimcolor{$d_4$}}};
    \draw[edge] (B) -- (\x+0.4,-0.3) node [midway, below] {\scalebox{0.7}{\dimcolor{$d_2$}}};
    \draw[edge] (B) -- (\x+1.5,-0.4)node [midway, below] {\scalebox{0.7}{\dimcolor{$d_3$}}};
\end{tikzpicture}
=
\begin{tikzpicture}[baseline=\baseline]
    \tikzset{tensor/.style = {minimum size = 0.5cm,shape = circle,thick,draw=black,inner sep = 0pt}, edge/.style   = {thick,line width=.4mm},every loop/.style={}}
    \def\y{0}
    \def\x{0}
  \node[tensor,fill=red!30!white] (D) at (\x-1,\y){$\Gt_1$};
  \node[tensor,fill=red!50!white] (A) at (\x,\y){\scalebox{0.85}{$\Gt_2$}};
  \node[tensor,fill=blue!50!white] (B) at (\x+1,\y){$\Gt_3$};
   \node[tensor,fill=blue!20!white] (C) at (\x+2,\y){$\Gt_4$};
    \draw[edge](A) -- (\x,\y-0.7)node [midway,left] {\edgelabel{$d_2$}};
    \draw[edge](B) -- (\x+1,\y-0.7)node [midway,left] {\edgelabel{$d_3$}};
    \draw[edge](C) -- (\x+2,\y-0.7)node [midway,left] {\edgelabel{$d_4$}};
     \draw[edge](D) -- (\x-1,\y-0.7)node [midway,left] {\edgelabel{$d_1$}};
    \draw[edge](D) -- (A) node [midway,above] {\edgelabel{$R_1$}};
    \draw[edge](A) -- (B) node [midway,above] {\edgelabel{$R_2$}};
    \draw[edge](B) -- (C) node [midway,above] {\edgelabel{$R_3$}};
\end{tikzpicture}.
\end{center}

Since $\Tt$ represents a probability distribution, its entries must be non-negative and sum to one.
In the context of machine learning, where we want to optimize the core tensors with respect to some loss function~(e.g. log-likelihood), there are several constraints one can impose on the core tensors $\Gt_1,\cdots,\Gt_N$ to ensure these two properties. 
One easy way to ensure non-negativity of the entries of $\Tt$ is to constrain the core tensors to have non-negative entries, or to apply a non-negative valued "activation" function component-wise to all the core tensors' entries. Another way, with its roots in quantum physics, is to assume that the entries of $\Tt$ are the square roots of the probabilities, rather than the probabilities themselves: $\Tt_{i_1,\cdots,i_N} = \sqrt{\PP(X_1=i_1,\cdots,X_N=i_N)}$. When $\Tt$ is in the TT format, this corresponds to the following tensor network:
$$
\PP(X_1=i_1,\cdots,X_4=i_4) = (\Tt_{i_1,\cdots,i_4})^2 = 

.
$$

Using the copy tensor, we can write this as:
\begin{align}\label{eqn:parametrized-prob}
    \PP(X_1,\cdots,X_4) \simeq
    %
\end{align}

To ensure that the probabilities sum to one, we can choose to model un-normalized squared-root probabilities, since, as we will see, computing the normalization constant is very easy when the probability tensor is in the TT format. We thus assume that
$$\PP(X_1=i_1,\cdots,X_N=i_N) = \frac{(\Tt_{i_1,\cdots,i_N})^2}{\zeta},$$
where $\zeta = \sum_{i_1,\cdots,i_N}(\Tt_{i_1,\cdots,i_N})^2$ is the normalization constant. Similarly to efficient operations in the TT format presented in Section~\ref{sec:efficient-operations}, this normalization factor can be efficiently computed in the TT format: 
$$\zeta= \sum_{i_1,\cdots,i_4}\Tt_{i_1,\cdots,i_4}^2
=
.
$$
The idea above is also called \emph{Born Machine}~\citep{glasser2019expressive,han2018unsupervised}.
Similarly to what we showed in Section~\ref{sub:probs}, we can compute marginal distributions of TT parametrized probability distributions~(see Eq.~\eqref{eqn:parametrized-prob}) as tensor networks:
\begin{enumerate}
    \item %
    Using the law of total probabilities, the marginal distribution of $X_2,X_3$, and $X_4$ can be expressed as 
    \begin{align*}
    \PP(X_2=i_2, X_3=i_3, X_4=i_4) 
    &= \sum_{i_1=1}^{d_1}\PP(X_1=i_1,\cdots,X_4=i_4)\\ 
    &=
    %
=
\sum_{i_1}^{d_1}\Tt_{i_1,\cdots,i_4},
    \end{align*}
for all $i_2\in[d_1],i_3\in[d_3]$ and $i_4\in[d_4]$. 
\item We can similarly express the remaining marginal distributions. For example, the marginal distribution of $X_2$ can be written as
\begin{align*}
    \PP(X_2=i_2) 
    &= \sum_{i_1,i_3,i_4=1}^{d_1,d_3,d_4}\PP(X_1=i_1,\cdots,X_4=i_4)\\ 
    &=
    %
=
\sum_{i_1,i_3,i_4=1}^{d_1,d_3,d_4}\Tt_{i_1,\cdots,i_4},
    \end{align*}
\end{enumerate}
\subsection{Computing Expectations of Random Tensor Networks} 
We will now show how the tensor networks formalism can help us derive complex properties of random tensors. 
We begin this section by presenting two very useful propositions. All of the results presented here are classical: the propositions and theorems are drawn primarily from~\citep{gupta2018matrix}. 
The first one is a trivial consequence of linearity of expectation, but will prove to be very useful in this section.
\begin{proposition}\label{prop:inner-expectation}
    Let $\At$ and $\Bt$ be two independent random tensor networks. Then their inner product in expectation is 
$$\EE\left(\begin{tikzpicture}[baseline=\baseline]
    \tikzset{tensor/.style = {minimum size = 0.5cm,shape = circle,thick,draw=black,inner sep = 0pt}, edge/.style = {thick,line width=.4mm},every loop/.style={}}
    \def\y{0}
    \def\x{0}
    \node[tensor,fill=red!50!blue!50!white] (At) at (\x,\y){\scalebox{0.85}{$\At$}};
    \node[tensor,fill=red!30!blue!50!white] (Bt) at (\x+1,\y){\scalebox{0.85}{$\Bt$}};
    \draw[edge] (At) -- (Bt);
    \draw[edge] (At) to [out=45,in=135] (Bt);
    \draw[edge] (At) to [out=-45,in=-135] (Bt);
    \draw[edge](At)--(\x-0.85,\y);
    \draw[edge](At)--(\x-0.75,\y+0.5);    \draw[edge](At)--(\x-0.75,\y-0.5);
    \draw[edge](Bt)--(\x+1.85,\y);
    \draw[edge](Bt)--(\x+1.75,\y+0.5);    \draw[edge](Bt)--(\x+1.75,\y-0.5);
\end{tikzpicture} \right)
=
\begin{tikzpicture}[baseline=\baseline]
    \tikzset{tensor/.style = {minimum size = 0.8cm,shape = circle,thick,draw=black,inner sep = 0pt}, edge/.style = {thick,line width=.4mm},every loop/.style={}}
    \def\y{0}
    \def\x{0}
    \node[tensor,fill=red!50!blue!50!white] (At) at (\x,\y){\scalebox{0.85}{$\EE(\At)$}};
    \node[tensor,fill=red!30!blue!50!white] (Bt) at (\x+1.5,\y){\scalebox{0.85}{$\EE(\Bt)$}};
    \draw[edge] (At) -- (Bt);
    \draw[edge] (At) to [out=45,in=135] (Bt);
    \draw[edge] (At) to [out=-45,in=-135] (Bt);
    \draw[edge](At)--(\x-0.85,\y);
    \draw[edge](At)--(\x-0.75,\y+0.5);    \draw[edge](At)--(\x-0.75,\y-0.5);
    \draw[edge](Bt)--(\x+2.35,\y);
    \draw[edge](Bt)--(\x+2.25,\y+0.5);    \draw[edge](Bt)--(\x+2.25,\y-0.5);
\end{tikzpicture}.
$$
\begin{proof}
The result directly follows from the linearity of the expected value and the independence of $\At$ and $\Bt$.
\end{proof}
\end{proposition}
The second one shows an interesting properties of matrices whose entries are drawn from a Gaussian distribution: the expectation of the outer product of such a matrix with itself reduces to a very simple tensor network. 
\begin{proposition}\label{prop:outer-prod}Let $\Ab\in\RR^{m\times n}$ be a random matrix whose elements are drawn identically and independently from the standard normal distribution. The expected value of the outer product of $\Ab$ with itself is 
$$\EE\left(
,
$$
where the right-hand side
is the outer product of two Kronecker delta~(i.e., $
%
=\delta_{i_1i_3}\delta_{i_2i_4}$ for $i_1,i_3\in[m]$ and $i_2,i_4\in[n]$).
\begin{proof}
Since the entries of $\Ab$ are drawn independently from $\mathcal{N}(0,1)$, its vectorization follows a multivariate normal distribution, $\vectorize(\Ab) \sim \mathcal{N}(\mathbf{0},\Ib_{mn})$, and thus $\EE(\vectorize(\Ab)\vectorize(\Ab)^\ts) = \Ib_{mn}$. This means that 
$$ \EE(\Ab_{i_1i_2}\Ab_{i_3i_4})=
\begin{cases}
        1 & \text{ if } i_1=i_3\text{ and } i_2=i_4\\
        0 & \text{ otherwise, }
    \end{cases} 
    $$
hence, $\EE(\Ab\outprod\Ab)_{i_1i_2i_3i_4}=\EE(\Ab_{i_1i_2}\Ab_{i_3i_4}) = \delta_{i_1i_3}\delta_{i_2i_4}=
$, for all $i_1,i_3\in[m]$ and $i_2,i_4\in[n]$.

\end{proof}
\end{proposition}
\begin{remark}
    Note that different orderings of indices of $\EE(\Ab\outprod\Ab)$ lead to the following tensor networks:  
    \begin{align*}
            \EE(\Ab\outprod\Ab)_{i_1i_2i_4i_3}
            &= 
            \EE(\Ab_{i_1i_2}\Ab_{i_4i_3}) 
    = \delta_{i_1i_4}\delta_{i_2i_3}\\
    &=
    \begin{cases}
        1 & \text{ if } i_1=i_4\text{ and } i_2=i_3\\
        0 & \text{ otherwise, }
    \end{cases} 
    =
    %
.
\end{align*}
\end{remark}
Next, we show how Propositions~\ref{prop:inner-expectation} and~\ref{prop:outer-prod} can be used to derive a very simple tensor network derivation of the expected value of the product of a random matrix with Gaussian entries with itself. 
\begin{proposition}\label{prop:wishart-mean}
If $\Ab\in\RR^{m\times n}$ is a random matrix whose elements are drawn i.i.d from the standard normal distribution, then $\EE(\Ab^\ts\Ab) = m\Ib_n$. 
\begin{proof}
        Using Propositions~\ref{prop:inner-expectation} and~\ref{prop:outer-prod} we can write 
\begin{align*}
\EE(\Ab^\ts\Ab) 
&= 
\EE\left(

= m\Ib_n.
\end{align*}
The final diagram is obtained by contracting the legs of size $m$. As observed, once the contraction is performed, the resulting expression is simply the trace of the identity matrix, which corresponds to the circle of size $m$ in the above equality.
\end{proof}
\end{proposition}
The previous proposition can be interpreted as a statement on the mean of a Wishart distribution.  
Recall that  sampling a matrix $\mat X$ from the Wishart distribution $W_p(\mat \Sigma, m)$, where $\mat \Sigma \in \RR^{n\times n}$ is a covariance matrix, is done by sampling each column of a matrix $\mat A\in\mathbb R^{m\times n}$ from the multivariate normal $\mathcal N(\vec 0 , \mat \Sigma)$ independently and setting $\mat X = \mat A^\top\mat A$. Hence, the matrix $\Ab ^\top \Ab$ in Proposition~\ref{prop:inner-expectation} follows the Wishart distribution  $W_p(\mat I,m)$.
\begin{remark} As a consequence of Proposition~\ref{prop:inner-expectation}, if $\Ab \in \RR^{m \times n}$ is a random matrix with entries drawn from the standard normal distribution, then $\EE(\norm{\Ab}_F^2) = mn$. 
This result can be extended to higher-order tensors whose entries are drawn i.i.d. from the standard normal distribution, e.g., if $\Tt\in\RR^{d_1\times d_2\times\cdots\times d_N}$ then $\EE(\norm{\Tt}_F^2) = d_1d_2\cdots d_N$.
\end{remark}
By leveraging the properties of random Gaussian matrices established in the previous propositions, we can compute the expected norm of their product. Once again, the proof is remarkably simple when expressed through tensor networks, whereas an explicit treatment using indices and summations would be considerably more intricate.
\begin{proposition}\label{prop:expected-norm}
    Let $\Ab\in\RR^{m\times r}$ and $\Bb\in\RR^{r\times n}$ be random matrices whose entries are independently and identically distributed (i.i.d.) according to the standard normal distribution with mean zero and variance one. Then, the expected value of the squared Frobenius norm of their product is given by $\EE({\norm{\Ab\Bb}^2_F}) = mnr$.
\begin{proof}
Using Proposition~\ref{prop:inner-expectation} we can write
\begin{align*}
 \EE({\norm{\Ab\Bb}^2_F})
 &= 
 \EE\left(
.
$ 
Since there are contractions between the legs of size $r$ we finally obtain
\begin{align*}
 \EE({\norm{\Ab\Bb}^2_F})
 = 
 \EE\left(%
= mnr.
\end{align*}
\end{proof}
\end{proposition}
Now that we understand how to compute the expected norm of products of Gaussian random matrices, we can extend this method to evaluate expectations of more complex tensor networks.

\paragraph{Example} Suppose, we have an arbitrary tensor network, e.g.,
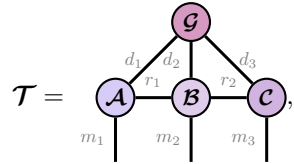
$$\Tt = 
%
,
$$
where $\At\in\RR^{m_1\times d_1\times r_1}, \Bt\in\RR^{m_2\times r_1\times d_2\times r_2}, \Ct\in\RR^{m_3\times r_2\times d_3}$ and $\Gt\in\RR^{d_1\times d_2\times d_3}$ are tensors whose elements are drawn i.i.d. from the standard normal distribution, and we want to compute the expectation of $\norm{\Tt}^2_F$. Using tensor network diagrams, we have
\begin{align*}
    \EE\left(\norm{%
\right).
\end{align*}
From Prop.~\ref{prop:wishart-mean}, we know that 

$$
\EE\left(%
,
$$
and thus by Proposition~\ref{prop:inner-expectation} we can write:
\begin{align*}
\EE\left(%
.
\end{align*}
Therefore, since $\Gt$ is independent of $\At,\Bt$ and $\Ct$ we can use Proposition~\ref{prop:inner-expectation} again  and write 
\begin{align*}
    \EE\left(%
\right)
\ \ \ \textequal{Prop.~\ref{prop:wishart-mean}}\ \  m_1m_2m_3r_1r_2d_1d_2d_3.
\end{align*}
We can see how efficiently complex computations can be performed using tensor networks. We now introduce other identities involving random Gaussian matrices.
\begin{enumerate}
\item  The first identity is a useful consequence of \emph{Isserlis'} theorem~\citep{isserlis1918formula}.
\begin{theorem}\label{thm:isserlis}
Let $\ab\in\RR^{n}$ be a random vector whose elements are drawn i.i.d. from the normal distribution with mean zero and variance one. Then
\begin{align*}
\EE(\ab^{\outprod 4}) 
&= 
\EE\left(

\end{align*}
where $\ab^{\outprod 4}$ denotes the outer product of $\ab$ with itself 4 times. In the last equality, each term in the sum corresponds to a different reshaping of the identity matrix.

This result can be extended to an arbitrary covariance structure: if $\ab$ is drawn from the multivariate normal $\mathcal N(\mathbf 0, \Sigmab )$, then
\begin{align*}
\EE(\ab^{\outprod 4})=
%
 .
\end{align*}

\begin{proof}
    We show the result for the case where the covariance is the identity. The proof extends straightforwardly to the case of an arbitrary covariance matrix. 
    
    Element-wise, by using Isserlis’
    theorem and using the fact that $\ab\sim\mathcal{N}(\mathbf{0},\Ib)$, for $i_1,\cdots,i_4\in[n]$, we have
\begin{align*}
(\EE(&\ab^{\outprod 4}))_{i_1,i_2,i_3,i_4} 
=
    \EE(\ab_{i_1}\ab_{i_2}\ab_{i_3}\ab_{i_4})\\
&= 
    \EE(\ab_{i_1}\ab_{i_2})\EE(\ab_{i_3}\ab_{i_4}) + \EE(\ab_{i_1}\ab_{i_3})\EE(\ab_{i_2}\ab_{i_4})
    +\EE(\ab_{i_1}\ab_{i_4})\EE(\ab_{i_2}\ab_{i_3})\\
&=
    \delta_{i_1i_2}\delta_{i_3i_4} + \delta_{i_1i_3}\delta_{i_2i_4} + \delta_{i_1i_4}\delta_{i_2i_3},
\end{align*}
which in tensor network notation gives:
\begin{align*}
\EE(\ab^{\outprod 4})_{i_1i_2i_3i_4}
=
.
\end{align*}
\end{proof}
\end{theorem}
We can also extend the result of \emph{Isserlis' theorem} to the case where the random variable takes the form of a higher-order tensor. For example, for a third-order tensor $\At\in\RR^{d_1\times d_2\times d_3}$, we can write
\begin{align*}
    (\EE(\At^{\outprod 4}))_{i_1j_1r_1i_2j_2r_2i_3j_3r_3i_4j_4r_4}
&=\\
&\hspace{-3cm}%
,
\end{align*}
for $i_1,\cdots,i_4\in[d_1], j_1,\cdots,j_4\in [d_2]$ and $r_1,\cdots,r_4\in[d_3]$.
\item Using the above results, we can easily derive the expectation of  $(\Ab^\ts\Ab)^{\outprod 2}$ when $\Ab$ is a random matrix:
\begin{proposition}\label{prop:kronecker-exp}
    Let $\Ab\in\RR^{m\times n}$ be a random matrix whose elements are drawn i.i.d from the standard normal distribution. Then
\begin{align*}
\EE((\Ab^\ts\Ab)^{\outprod 2})) 
=
m^2%
.
\end{align*}
\begin{proof}
By applying Theorem~\ref{thm:isserlis} to matrices and using Proposition~\ref{prop:wishart-mean}, we can derive the desired expression. First, from Theorem~\ref{thm:isserlis} we have
\begin{align*}
\EE(\Ab^{\outprod4})_{i_1j_1i_2j_2i_3j_3i_4j_4}
&= 
\EE\left(%
,
\end{align*}
for $i_1,\cdots, i_4\in[n]$ and $j_1,\cdots,j_4\in [m]$. It follows that
\begin{align*}
\EE((\Ab^\ts\Ab)^{\outprod 2})_{i_1i_2i_3i_4}
&=
\EE\left(%
.
\end{align*}
\end{proof}
\end{proposition}
\item More generally, if $\Ab\in\RR^{m\times n}$ is a random matrix with i.i.d. standard normal entries and $\Tt\in\RR^{n\times n\times n\times n}$ is an arbitrary random tensor independent of $\Ab$ , then contracting $\Tt$ with $\Ab$ yields
\begin{align*}
&\hspace*{-2cm}\EE\left(%
\right)= \sum_{i_1i_2i_3i_4}\EE(\Tt_{i_1i_2i_3i_4})(m^2\delta_{i_1i_2}\delta_{i_3i_4}+m\delta_{i_1i_3}\delta_{i_2i_4}+m\delta_{i_1i_4}\delta_{i_2i_3}).
\end{align*}
\item Let $\Ab\in\RR^{m\times n}$ and $\Bb\in\RR^{n\times p}$ be two random matrices whose elements are drawn i.i.d from the standard normal distribution, then
$$\EE(\Ab\Bb)^{\outprod 2} = 
\EE\left(%
.
$$
\begin{proof}
    Using Proposition~\ref{prop:inner-expectation} and~\ref{prop:outer-prod} we can write:
\begin{align*}
\EE\left(%
.
\end{align*}
\end{proof}
\item The next proposition is a useful special case of Theorem 3.3.15 item (iv) of~\citep{gupta2018matrix}.
\begin{proposition}
    Let $\Ab\in\RR^{m\times n}$ be a random matrix with i.i.d. standard normal entries, and let $\Xb\in\RR^{d\times m}$ be a constant matrix then 
$$\EE(\tr(\Xb\Ab\Ab^\ts\Xb^\ts)^2) = n^2\norm{\Xb}_F^4 + 2n\tr((\Xb^\ts\Xb)^2).$$
\begin{proof}
    We prove this identity by using tensor network notations:
\begin{align*}
\EE(\tr(\Xb\Ab\Ab^\ts\Xb^\ts)^2) 
&=
\EE\left(%
,
\end{align*}
The last equality follows from the fact that the matrix $\Xb$ is constant. By Proposition~\ref{prop:kronecker-exp} we know:
\begin{align*}
\EE(\Ab^\ts\Ab)^{\outprod 2}
&=
n^2%
\\
&= 
n^2\norm{\Xb}_F^4+2n\tr((\Xb^\ts\Xb)^2).
\end{align*}
\end{proof}
\end{proposition}
\end{enumerate}

The derivations we presented in this chapter demonstrate how manipulating tensor network expressions can help us derive non-trivial results in a very intuitive manner.

\section{Conclusion}\label{chapter:conclusion}

This manuscript presents a self-contained graphical language for tensor networks, progressing from foundational notation and basic concepts to tensor decompositions, gradient computation, and probability distributions. Its primary objective is to provide practical frameworks for performing operations, ranging from simple to complex, on high-dimensional objects.
A central theme throughout the manuscript is that the graphical language of tensor
networks is not merely a notational convenience. Rather, it offers a clear and precise mathematical framework that brings together many classical results into a single, unified structure.
Several examples illustrate this point. The bounds on the matricization rank are determined by the weight of any cut that separates the row legs from the column legs in a tensor diagram (Theorem~\ref{thm: matricization-rank}). The gradient rules of Chapter~\ref{chapter:gradients}~(Theorems~\ref{thm:gradients} and~\ref{thm:multiple-tensors-gradients}) reduce the computation
of Jacobians of complex tensor network expressions to a single diagrammatic operation,
node removal, recovering classical identities such as
$\partial(\Ab\xb)/\partial \Ab = \xb \outprod \Ib$ and $\partial \tr(\Ab)/\partial \Ab = \Ib$
as immediate corollaries. The probabilistic computations of Chapter~\ref{chapter:prob} show that
non-trivial results, including the mean of the Wishart distribution (Proposition~\ref{prop:wishart-mean}), Isserlis' theorem for Gaussian tensors~(Theorem~\ref{thm:isserlis}), and the expected norm of the product of Gaussian random matrices (Proposition~\ref{prop:expected-norm}), admit concise derivations using tensor network notation.

Taken together, these examples highlight a general principle: when working with
high-dimensional tensors, organizing computations as tensor network contractions
and reasoning diagrammatically leads to representations that are both compact and
conceptually transparent.

\section*{Acknowledgments}
The authors thank Jacob Miller for fruitful discussions and for being part of the initial motivation to write this manuscript.
We would like to gratefully acknowledge the Canadian Institute for Advanced Research (CIFAR AI chair program), the Natural Sciences and Engineering
Research Council of Canada (Discovery program, RGPIN-2019-05949) and IVADO for funding this research.

\newpage
\bibliography{biblio}
\end{document}